\documentclass[a4paper,runningheads]{llncs}
\ifdefined\pdfminorversion
\fi
\usepackage[T1]{fontenc}
\usepackage{graphicx}
\usepackage{amsmath}
\usepackage{amssymb}
\usepackage{booktabs}
\usepackage{xurl}
\usepackage{algorithm}
\usepackage{algpseudocode}
\AtBeginDocument{
  \setlength{\abovedisplayskip}{6pt plus 2pt minus 2pt}
  \setlength{\belowdisplayskip}{6pt plus 2pt minus 2pt}
  \setlength{\abovedisplayshortskip}{3pt plus 1pt minus 1pt}
  \setlength{\belowdisplayshortskip}{3pt plus 1pt minus 1pt}
}
\allowdisplaybreaks
\usepackage{hyperref}
\usepackage{color}

\usepackage[numbers]{natbib}

\begin{document}
\title{Parameter Learning in Active Inference via Variational Message Passing}
\titlerunning{Parameter Learning in Active Inference via Variational Message Passing}
\newcommand{\at}{\char64}
\newcommand{\period}{\char46}
\newcommand{\myemail}{\texttt{benedikt\period{}waldeck\at{}unimore\period{}it}}
\author{Benedikt Waldeck\inst{1}\orcidID{0009-0001-5225-9933}}
\authorrunning{B. Waldeck}
\institute{University of Modena and Reggio Emilia, Italy\\
\email{\myemail}}
\maketitle
\begin{abstract}
Active inference casts perception, learning and planning as variational inference, where the Expected Free Energy (EFE) defines the planning objective. Recent work showed that EFE minimization can be written as Variational Free Energy (VFE) minimization on a generative model augmented with epistemic priors, and that the resulting entropy-corrected objective is minimized by sum-product message passing (SP-MP) on modified factor functions. These formulations treat model uncertainty as \emph{inference} over a finite set of known models. We extend them to planning during \emph{learning}, where beliefs over the observation and transition parameters $\theta_A$, $\theta_B$ are Dirichlet priors. Closed-form Dirichlet updates are usually obtained by a structured mean-field factorization between parameters and states, but under exactly that factorization the information-gain corrections vanish. We resolve this by giving every factor its own parameter prior, frozen during a rollout. Conjugacy then makes each factor belief a product of a marginal over the categorical variables and a one-step Bayes posterior of the parameters. This reintroduces the datum-dependent parameter posteriors that mean-field factorization destroys. We show that evaluating the epistemic priors under the resulting Bethe joint recovers the information gains. Stationarity of this single closed-form objective yields SP-MP whose factor functions are augmented by the information-gain weight. Experiments on two grid POMDPs, in which only a planner that learns the model finds the safe route, show that the method plans to efficiently learn $\theta_A$ and $\theta_B$. It reaches the goal on every episode, whereas SP-MP planners with the same inference update plateau or take the risky route.

\end{abstract}
\section{Introduction}
\label{sec:introduction}

Active inference casts perception, planning and learning as variational inference on a single generative model~\cite{Da_Costa_2020}. Its planning objective, the Expected Free Energy (EFE), decomposes into risk, ambiguity and novelty and thereby balances goal pursuit against epistemic exploration, but evaluating it by enumerating policies scales exponentially with the horizon~\cite{friston2016active}. De Vries et al.~\cite{devries2025expectedfreeenergybasedplanning} showed that EFE minimization can be reformulated as Variational Free Energy (VFE) minimization on a generative model augmented with \emph{epistemic priors}. Nuijten et al.~\cite{nuijten2026whattypeinference} showed that the augmented VFE equals the predictive VFE plus explicit \emph{entropy corrections}, including the planning correction of Lázaro-Gredilla et al.~\cite{lázarogredilla2025typeinferenceplanning}. They derived a message passing scheme in which each correction becomes a \emph{channel reparameterization} of a factor function, so that EFE-based planning runs as sum-product message passing (SP-MP) on the original factor graph.

These formulations treat the model parameter $\theta$ as a single enumerable latent variable: they plan under uncertainty over a finite set of known models, not under uncertainty about the models themselves. Planning during \emph{Learning} the observation model $\theta_A$ and the transition model $\theta_B$ conjugately, an agent that knows neither its sensors nor its actuators, is left open and poses a structural problem. Closed-form Dirichlet updates require a mean-field factorization between the parameter beliefs and the state--observation blocks, but under exactly that factorization every parameter-posterior contrast collapses, $q(\theta\mid\cdot)=q(\theta)$, and the novelty term that motivates parameter learning vanishes from the objective.

This paper resolves this problem within a single variational objective. We split $\theta$ into an observation block $\theta_A$, a transition block $\theta_B$, and an enumerable context $\theta_c$, and extend the entropy-corrected formulation to conjugate learning of $\theta_A$ and $\theta_B$. Our contributions are:
\begin{itemize}
    \item We show that the channels of the entropy corrections need not be free variational parameters. Substituting them at their closed-form optima, the conditionals of $q$, and giving each factor a step-wise parameter prior frozen at the current belief, keeps every sum-product factor belief conjugate and preserves the local parameter–outcome joints that mean-field destroys. Under the resulting Bethe joint the objective is a single scalar over the discrete coordinates.
    \item We derive the resulting message passing scheme. Four epistemic priors recover the EFE with separate novelty terms for $\theta_A$ and $\theta_B$ and a salience term for $\theta_c$. Stationarity yields belief propagation on the discrete graph with factor functions augmented with per-entry \emph{information-gain weights} $e^{D_{\mathrm{KL}}}\ge1$. The fixed-parameter scheme of~\cite{nuijten2026messagepassingrealizationexpected,nuijten2026whattypeinference} is recovered in the point-mass limit.
    \item We validate the method on two grid POMDPs, in which only a planner that explores the model, finds the safe route. The EFE-planner learns $\theta_A$ and $\theta_B$ and reaches the goal on every episode, whereas SP-MP baselines plateau on parameter learning or take the risky routes. Enumerated EFE reproduces its behaviour, and ablations isolate the novelty terms as the drive.
\end{itemize}
Section~\ref{sec:background} reviews the EFE and the Bethe free energy. Section~\ref{sec:related} discusses related work. Section~\ref{sec:methodology} derives the objective and the message passing scheme. Section~\ref{sec:experiments} reports the experiments and Section~\ref{sec:conclusion} concludes.

\section{Background}
\label{sec:background}

\subsection{Active Inference and the Expected Free Energy}
\label{sec:ai_efe}

We consider an agent that maintains a generative model predicting future observations, states and the consequences of actions over a planning horizon $T$. For a discrete partially observable Markov decision process (POMDP) with hidden states $\boldsymbol{x}$, observations $\boldsymbol{y}$, actions $\boldsymbol{u}$ and model parameters $\theta$. We write this as a rollout model $p(\boldsymbol{y},\boldsymbol{x},\boldsymbol{u},\theta) = p(\theta)\,p(x_0)\prod_{t=1}^{T} p(u_t)\,p(y_t\mid x_t,\theta)\,p(x_t\mid x_{t-1},u_t,\theta)$. To encode goals, the model is augmented with \emph{preference priors} $\hat{p}(x_t)$, $\hat{p}(y_t)$ over desired states and observations, which can be understood as exponentiated rewards $\hat{p}\propto\exp(R)$~\cite{Da_Costa_2020}.

Active inference selects actions by minimizing the Expected Free Energy (EFE), which scores a policy by the cost of its predicted trajectory and decomposes into three interpretable terms~\cite{Da_Costa_2020}. Per time step,
\begin{equation}
\label{eq:efe_decomp}
\begin{aligned}
G(u_t) = \underbrace{D_{\mathrm{KL}}\!\big[q(x_t\mid u_t)\,\|\,\hat{p}(x_t)\big]}_{\text{risk (states)}}
&+ \underbrace{\mathbb{E}_{q(x_t\mid u_t)}\!\big[\mathbb{H}[p(y_t\mid x_t)]\big]}_{\text{ambiguity}} \\
&- \underbrace{\mathbb{E}_{p(y_t\mid x_t)\,q(x_t\mid u_t)}\!\big[D_{\mathrm{KL}}[q(\theta\mid y_t,x_t)\,\|\,q(\theta)]\big]}_{\text{novelty}}.
\end{aligned}
\end{equation}
Risk penalizes deviation of the predicted states from the preferences, ambiguity penalizes states whose observations are uncertain, and novelty rewards states and observations that are informative about $\theta$. In discrete active inference, policies (sequences of actions) are typically scored with~\eqref{eq:efe_decomp} and selected in proportion to $\exp(-G)$~\cite{friston2016active}.

\subsection{Variational Free Energy and Message Passing on Factor Graphs}
\label{sec:vi_ffg}

Given a generative model $p$, variational inference approximates its posterior by minimizing the Variational Free Energy (VFE) over a family of tractable distributions $q$:
\begin{equation}
\label{eq:vfe}
F[q] = D_{\mathrm{KL}}\!\big[q(\boldsymbol{y},\boldsymbol{x},\boldsymbol{u},\theta)\,\|\,p(\boldsymbol{y},\boldsymbol{x},\boldsymbol{u},\theta)\big].
\end{equation}
The generative model factorizes into local terms and can be represented as a Forney-style factor graph (FFG)~\cite{senoz2021variational}, in which nodes $a\in\mathcal{V}$ represent factors $f_a$, edges $i\in\mathcal{E}$ represent variables, and $d_i$ is the degree of edge $i$. The \emph{Bethe approximation} constrains $q$ to respect this structure: each node maintains a belief $q_a(\boldsymbol{s}_a)$ over its adjacent variables and each edge a singleton belief $q_i(s_i)$, with
\begin{equation}
\label{eq:bethe_approx}
q(s)=\frac{\prod_{a} q_a(\boldsymbol{s}_a)}{\prod_{i} q_i(s_i)^{d_i-1}},
\qquad
\mathbb{H}_{\mathrm{Bethe}}[q]=\sum_{a}\mathbb{H}[q_a]-\sum_{i}(d_i-1)\,\mathbb{H}[q_i],
\end{equation}
so that the joint entropy $\mathbb{H}[q]$ is approximated by $\mathbb{H}_{\mathrm{Bethe}}[q]$ and the VFE reduces to the \emph{Bethe free energy}
\begin{equation}
\label{eq:bethe}
F_{\mathrm{Bethe}}[q] = \sum_{a\in\mathcal{V}} D_{\mathrm{KL}}\!\big[q_a(\boldsymbol{s}_a)\,\|\,f_a(\boldsymbol{s}_a)\big] + \sum_{i\in\mathcal{E}} (d_i-1)\,\mathbb{H}[q_i(s_i)].
\end{equation}
The beliefs must be normalized and locally consistent, so minimizing~\eqref{eq:bethe} is a constrained problem whose Lagrangian stationarity conditions yield message passing: sum-product message passing (belief propagation) under the consistency constraints alone, and Variational Message Passing (VMP) when node beliefs are further mean-field factorized~\cite{senoz2021variational}.

\section{Related Work}
\label{sec:related}

\paragraph{Planning-as-inference and active inference.} Active inference inherits the planning-as-inference framework~\cite{attias2003planning,toussaint2009robot,levine2018reinforcement} and extends it by treating information gain as an intrinsic objective, so that the EFE combines pragmatic and epistemic value~\cite{friston2015active,Da_Costa_2020}. Existing EFE planners employ specialized procedures: policy tree search~\cite{friston2021sophisticated}, branching-time recursion~\cite{champion2022branching}, dynamic programming~\cite{paul2024dpefe}, or message passing on a \emph{generalized} free energy~\cite{parr2019generalised} with a mutual-information term over future graph segments~\cite{vandelaar2024realizing}; none minimizes a single variational objective over the whole graph.

\paragraph{EFE minimization as VFE minimization.} Control-as-inference~\cite{kappen2012optimal,levine2018reinforcement} augments $p$ with preference priors $\hat{p}(x_t)$ and, where used, $\hat{p}(y_t)$, giving a preference-biased model $\hat{p}$ and the VFE $F_{\hat{p}}[q]=\mathbb{E}_q[\log q/\hat{p}]$. De Vries et al.~\cite{devries2025expectedfreeenergybasedplanning} (Proposition~1) showed that further augmenting $\hat{p}$ with \emph{epistemic priors} $\tilde{p}$ makes this VFE equal to $\mathbb{E}_{q(\boldsymbol{u})}[G(\boldsymbol{u})]$ up to a complexity term. Nuijten et al.~\cite{nuijten2026whattypeinference} showed that the augmented VFE can equivalently be written as $F_{\hat{p}}$ plus explicit \emph{entropy corrections}, one conditional-entropy term per prior: $\tilde{p}(u_t)$ contributes $-\mathbb{H}[q(x_t\mid x_{t-1},u_t)]$ and the observation-side priors contribute $+\mathbb{H}[q(y_t\mid x_t,\theta)]$ terms. Proper planning over policies requires one further \emph{planning correction} $+\mathbb{H}[q(u_t\mid x_{t-1})]$, introduced by Lázaro-Gredilla et al.~\cite{lázarogredilla2025typeinferenceplanning} as the correction that turns marginal inference into policy optimization~\cite{nuijten2026whattypeinference}.

\paragraph{Message passing realizations.} The conditional entropies above are ratios of region beliefs and therefore not independent coordinates of the Bethe free energy. Nuijten et al.~\cite{nuijten2026messagepassingrealizationexpected} attached the epistemic priors to a parameter-free model as extra factor nodes and resolved their circular dependence on $q$ by alternating posterior and prior updates. Nuijten et al.~\cite{nuijten2025activeinferencesubtypevariational} introduced region-extended coordinates and a single observation-side channel for $\mathbb{H}[q(y_t\mid x_t,\theta)]$, without the planning correction. Nuijten et al.~\cite{nuijten2026whattypeinference} finally absorbed every correction, including the planning correction, into modified factor functions through channel reparameterization, recovering belief propagation, cross-entropy planning and full EFE-based planning as one family. The key identity of that construction is the variational characterization of conditional entropy (Gibbs' inequality),
\begin{equation}
\label{eq:channel_identity}
\mathbb{H}[q(y\mid x)] = \min_{r}\,\mathbb{E}_{q(y,x)}\!\big[-\log r(y\mid x)\big],
\end{equation}
with equality at the normalized \emph{channel} $r(y\mid x)=q(y\mid x)$. Treating each channel as a free variational parameter turns a single factor into a modified \emph{factor function}, and the corrected Bethe objective is minimized by sum-product updates on those functions, jointly over beliefs and channels.

\paragraph{This paper.} Every realization above assumes a single, enumerable $\theta$, and carries the channels of~\eqref{eq:channel_identity} as free variational parameters. We instead let $\theta$ carry Dirichlet beliefs over the model parameters themselves, and substitute the channels at their closed-form optima, where conjugate factor beliefs supply the parameter--datum coupling that the channels would otherwise have to represent.

\section{Methodology}
\label{sec:methodology}

\subsection{Generative Model and Planning Problem}
\label{sec:generative_model}

We consider a discrete-time POMDP with hidden states $x_t$, actions $u_t$, observations $y_t$ and an episode-level parameter $\theta=(\theta_A,\theta_B,\theta_c)$ in three blocks,
\begin{equation}
    \label{eq:generative_model}
    \begin{gathered}
    p(\boldsymbol{y},\boldsymbol{x},\boldsymbol{u},\theta) = p(\theta_A)\,p(\theta_B)\,p(\theta_c)\,p(x_0)
    \prod_{t=1}^{T} p(u_t)\,f_{\mathrm{dyn},t}\,f_{\mathrm{obs},t},\\
    f_{\mathrm{dyn},t}=p(x_t \mid x_{t-1}, u_t, \theta_B,\theta_c),\qquad f_{\mathrm{obs},t}=p(y_t \mid x_t, \theta_A,\theta_c).
    \end{gathered}
\end{equation}
The context $\theta_c\in\{1,\dots,K\}$ is Categorical and selects which of $K$ observation and transition matrices apply; $\theta_A,\theta_B$ represent those matrices, $p(y_t\mid x_t,\theta_A,\theta_c)=\theta_A(y_t\mid x_t,\theta_c)$ and $p(x_t\mid x_{t-1},u_t,\theta_B,\theta_c)=\theta_B(x_t\mid x_{t-1},u_t,\theta_c)$, with one Dirichlet prior per column, $p(\theta_A)=\prod_{x,\theta_c}\mathrm{Dir}(\alpha_A(\cdot\mid x,\theta_c))$ and $p(\theta_B)=\prod_{x,u,\theta_c}\mathrm{Dir}(\alpha_B(\cdot\mid x,u,\theta_c))$. Unknown $\theta_A,\theta_B$ means unlearned matrices; unknown $\theta_c$ means known matrices but an unknown context, the enumerable model variable of~\cite{nuijten2026whattypeinference}. One draw of $\theta$ is shared by all $T$ factors. Goals enter as preference priors, giving the preference-biased model $\hat{p}=p\prod_t\hat{p}_x(x_t)\,\hat{p}_y(y_t)$ (Figure~\ref{fig:ffg}) with VFE $F_{\hat{p}}[q]=\mathbb{E}_q[\log q/\hat{p}]$ ($\hat{p}$ unnormalized, so a divergence up to a constant).

\begin{figure}[t]
\centering
\includegraphics[trim=42 42 42 42, clip, width=0.98\textwidth]{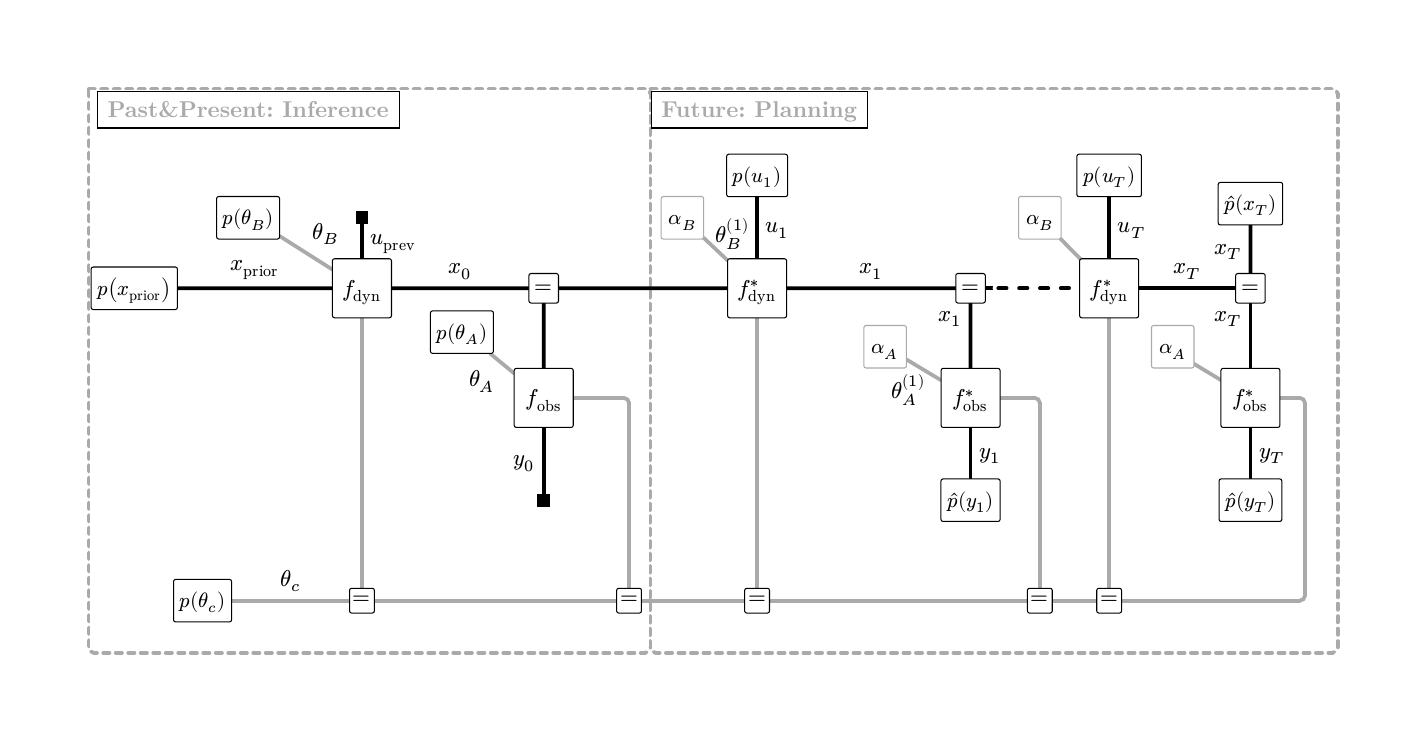}
\caption{\textbf{Frozen-parameter factor graph} (Section~\ref{sec:family}). Squares: factors; edges: variables; $=$: equality nodes; filled squares: clamped data; $\hat{p}$: preferences. \emph{Left:} the real slice carries the shared $\theta_A,\theta_B$ and turns $p(x_{\mathrm{prior}}),u_{\mathrm{prev}},y_0$ into the planning prior $p(x_0)$. \emph{Right:} over the horizon each factor $f^{*}$ (factor functions~\eqref{eq:kernel_obs}--\eqref{eq:kernel_dyn}) owns a copy $\theta_A^{(t)},\theta_B^{(t)}$ with its own degree-two prior node (open, gray) frozen at $\mathrm{Dir}(\alpha_A)$ resp.\ $\mathrm{Dir}(\alpha_B)$; only $\theta_c$ is shared, through equality nodes.}
\label{fig:ffg}
\end{figure}

The agent scores an action by the per-step EFE,
\begin{equation}
    \label{eq:efe}
    \begin{aligned}
    G(u_t) = \mathbb{E}_{q(\cdot \mid u_t)}\!\Big[&\log \frac{q(x_t \mid x_{t-1}, u_t)}{\hat{p}_x(x_t)\,\hat{p}_y(y_t)\,p(y_t \mid x_t, \theta_A,\theta_c)}\\
    &+\log\frac{q(\theta_A,\theta_c \mid x_t)}{q(\theta_A,\theta_c \mid y_t, x_t)}
    +\log\frac{q(\theta_B \mid x_{t-1},\theta_c)}{q(\theta_B \mid x_t, x_{t-1}, u_t,\theta_c)}\Big],
    \end{aligned}
\end{equation}
where $q(\cdot\mid u_t)$ abbreviates $q(y_t,x_t,x_{t-1},\theta\mid u_t)$. The first term is risk plus ambiguity $\mathbb{H}_q[y_t\mid x_t,\theta_A,\theta_c]$. The remaining terms contrast each block's \emph{cavity}, the belief before the imagined datum, with its \emph{posterior}, and are information gains. By the chain rule they split into the \emph{salience} $\mathcal{I}_{c,t}=I(\theta_c;y_t\mid x_t)$, which rewards resolving the context, and the \emph{novelties} $\mathcal{I}_{A,t}=I(\theta_A;y_t\mid x_t,\theta_c)$ and $\mathcal{I}_{B,t}=I(\theta_B;x_t,u_t\mid x_{t-1},\theta_c)$, which reward learning the matrices. The planning problem is to find the policy $q(u_t\mid x_{t-1})$ that minimizes $\sum_t\mathbb{E}_{q(u_t)}[G(u_t)]$.

Two obstacles separate~\eqref{eq:efe} from a message-passing algorithm. First, $G$ contains the agent's own conditionals inside the logarithm, so it is not a VFE against any fixed model. Second, the novelties are mutual informations of \emph{local joints} such as $q(y_t,x_t,\theta_c,\theta_A)$; a family that factorizes the Dirichlet block away from the imagined datum, $q(y_t,x_t,\theta_c)\,q(\theta_A)$, has $\mathcal{I}_{A,t}\equiv0$, yet this is exactly the factorization that structured mean-field uses to obtain Dirichlet conjugacy. Section~\ref{sec:objective} fixes the first, Section~\ref{sec:family} the second.

\subsection{From Expected to Variational Free Energy}
\label{sec:objective}

Following~\cite{devries2025expectedfreeenergybasedplanning,nuijten2026whattypeinference}, we keep the objective and change the model: $\hat{p}$ is augmented with \emph{epistemic priors}, one for every factor of~\eqref{eq:efe} that is not already part of $\hat{p}$,
\begin{align}
    \tilde{p}(u_t) &\propto \exp\!\big(-\mathbb{E}_{q(x_t,x_{t-1}\mid u_t)}\log q(x_t\mid x_{t-1},u_t)\big), \label{eq:prior_u}\\
    \tilde{p}(x_t) &\propto \exp\!\big(-\mathbb{E}_{q(\theta_A,\theta_c\mid x_t)}\,h[q(y_t\mid x_t,\theta_A,\theta_c)]\big), \label{eq:prior_x}\\
    \tilde{p}(y_t,x_t) &\propto \exp D_{\mathrm{KL}}[q(\theta_A,\theta_c\mid y_t,x_t)\,\|\,q(\theta_A,\theta_c\mid x_t)], \label{eq:prior_yx}\\
    \tilde{p}(x_t,x_{t-1},u_t,\theta_c) &\propto \exp D_{\mathrm{KL}}[q(\theta_B\mid x_t,x_{t-1},u_t,\theta_c)\,\|\,q(\theta_B\mid x_{t-1},\theta_c)], \label{eq:prior_xxu}
\end{align}
where $h[q(\cdot\mid x)]$ denotes the entropy of a conditional as a function of $x$ and $\mathbb{H}_q[\,\cdot\mid x]=\mathbb{E}_{q(x)}h[q(\cdot\mid x)]$ the conditional entropy. Each prior is the exponential of a local functional of $q$, so $-\mathbb{E}_q[\log\tilde{p}]$ is closed-form.

\begin{corollary}[Epistemic priors for three parameter blocks; cf.~\cite{devries2025expectedfreeenergybasedplanning}, Prop.~1]
\label{cor:dual_priors}
Let $\tilde{p}$ be $\hat{p}$ augmented with the priors~\eqref{eq:prior_u}--\eqref{eq:prior_xxu} for $t=1,\dots,T$. Then $F_{\tilde{p}}[q] = \sum_{t}\mathbb{E}_{q(u_t)}[G(u_t)] + B[q]$, where $B[q]=D_{\mathrm{KL}}[q\,\|\,p]\ge0$ is the complexity of $q$ relative to~\eqref{eq:generative_model}.
\end{corollary}
\begin{proof}
See Appendix~\ref{app:priors_efe}.
\end{proof}

The prior~\eqref{eq:prior_xxu} is the transition-novelty prior, new relative to~\cite{nuijten2026whattypeinference}. Proper planning over policies further requires the planning correction $+\mathbb{H}_q[u_t\mid x_{t-1}]$~\cite{lázarogredilla2025typeinferenceplanning,nuijten2026whattypeinference}, which cancels the reward that the VFE pays for action entropy. Evaluating $-\mathbb{E}_q[\log\tilde{p}]$ for each prior (Appendix~\ref{app:priors_entropy}) writes the EFE-based planning objective as the base VFE plus a per-step \emph{entropy correction}, $F_T[q] = F_{\hat{p}}[q] + \sum_{t=1}^{T} \Delta F_t[q]$ (the subscript is the horizon) with
\begin{equation}
    \label{eq:deltaF}
    \Delta F_t[q] = \mathbb{H}_q[y_t \mid x_t] - 2\big(\mathcal{I}_{A,t}+\mathcal{I}_{c,t}\big) - \mathbb{H}_q[x_t \mid x_{t-1}, u_t] - \mathcal{I}_{B,t} + \mathbb{H}_q[u_t \mid x_{t-1}].
\end{equation}
The factor two is no double count: the ambiguity prior~\eqref{eq:prior_x} pays one unit of each observation-side gain through its $\theta$-coupling and the novelty prior~\eqref{eq:prior_yx} a second, while the transition side has no ambiguity term (Appendix~\ref{app:priors_entropy}). Message passing on $\tilde{p}$ now applies, \emph{provided} that $q$ retains the local joints that make the information gains in~\eqref{eq:deltaF} nonzero.

\subsection{Variational Family: Sum-Product Beliefs on the Frozen-Parameter Graph}
\label{sec:family}

\paragraph{Frozen-parameter graph.} The agent plans with its current parameter posterior, so imagined observations in the rollout must not update it. On the graph of~\eqref{eq:generative_model} this fails: the sum-product message from $f_{\mathrm{obs},t}$ to a shared $\theta_A$ node, $\sum_{y,x,\theta_c}p(y\mid x,\theta_A,\theta_c)\,\mu_y\mu_x\mu_{\theta_c}$ with $\mu$ the messages into the factor, is a mixture of monomials rather than a Dirichlet, so conjugacy is lost and imagined observations mix across the horizon. Both problems are resolved by a step-wise copy $\theta_A^{(t)}$ per observation factor with its own prior node $p(\theta_A^{(t)}):=q_{\theta_A}=\mathrm{Dir}(\theta_A^{(t)}\mid\alpha_A)$ carrying the current real-data belief, and likewise $\theta_B^{(t)}$ per $f_{\mathrm{dyn},t}$ (Figure~\ref{fig:ffg}; $q_{\theta_A},q_{\theta_B}$ denote these beliefs, as opposed to the copy marginals). Setting a prior node to a belief makes $\tilde{p}$ a \emph{planning-time} model, rebuilt each environment step, so the objective is single per planning call. Each copy has degree two, so the message into its factor is exactly $\mathrm{Dir}(\alpha_A)$ and the return message reaches no other factor; the price is that novelty does not decay along the horizon, which replanning bounds (Appendix~\ref{app:mp_generic}). The finite $\theta_c$ is \emph{not} copied: it is conjugate to any factor, and sharing it lets an early cue inform a distant step.

\paragraph{Factor beliefs are conjugate in closed form.} A Bethe factor belief is the factor times its incoming messages, $q_{\mathrm{obs},t}\propto p(y_t\mid x_t,\theta_A^{(t)},\theta_c)\,\mu_{y_t}\mu_{x_t}\mu_{\theta_c}\,\mathrm{Dir}(\theta_A^{(t)}\mid\alpha_A)$. For each context the likelihood is a monomial in one column of $\theta_A$. Thus conjugacy applies entrywise, $\theta_A(y\mid x,\theta_c)\,\mathrm{Dir}(\theta_A\mid\alpha_A)=m_A(y\mid x,\theta_c)\,\mathrm{Dir}(\theta_A\mid\alpha_A+e_{yx\theta_c})$, with $e_{yx\theta_c}$ one count on entry $y$ of column $(x,\theta_c)$ and $m_A=\alpha_A/\alpha_{A,0}$ the posterior predictive likelihood (Appendix~\ref{app:family}). Hence, exactly,
\begin{equation}
    \label{eq:factor_belief}
    q_{\mathrm{obs},t}(y_t,x_t,\theta_c,\theta_A^{(t)})=q_{yx,t}(y_t,x_t,\theta_c)\;\underbrace{\mathrm{Dir}(\theta_A^{(t)}\mid\alpha_A+e_{y_tx_t\theta_c})}_{q_{\mathrm{obs},t}(\theta_A\mid x_t,y_t,\theta_c)},
\end{equation}
with discrete marginal $q_{yx,t}:=\!\int\! q_{\mathrm{obs},t}\,d\theta_A\propto m_A\,\mu_{y_t}\mu_{x_t}\mu_{\theta_c}$, and $q_{\mathrm{dyn},t}=q_{\mathrm{trip},t}(x_t,x_{t-1},u_t,\theta_c)\,\mathrm{Dir}(\theta_B\mid\alpha_B+e_{x_tx_{t-1}u_t\theta_c})$ with the predictive likelihood $m_B$ of $\alpha_B$.

\paragraph{Coordinates and constraints.} The free coordinates are the Bethe beliefs of the discrete graph over $\{y_{1:T},x_{0:T},u_{1:T},\theta_c\}$: the block beliefs $q_{yx,t}$, $q_{\mathrm{trip},t}$ and the singletons $q(x_t)$, $q(y_t)$, $q(u_t)$, $q(\theta_c)$, subject to normalization and factor-to-singleton consistency, with $q_{\theta_A}=\mathrm{Dir}(\alpha_A)$ and $q_{\theta_B}=\mathrm{Dir}(\alpha_B)$, held fixed during planning. The $\theta_A$- and $\theta_B$-conditionals of the factor beliefs are not coordinates: any other choice raises the free energy of the region, so they are substituted out at the closed-form optima of~\eqref{eq:factor_belief} (Appendix~\ref{app:family}, Lemma~\ref{lem:app_elimination}). There are no channel constraints either: the channels of~\eqref{eq:channel_identity} are not coordinates but definitions $r=q(\cdot\mid\cdot)$.

\paragraph{Cavities and per-entry information gains.} Novelty contrasts the posterior $q_{\mathrm{obs},t}(\theta_A\mid x_t,y_t,\theta_c)$ with the \emph{cavity} $q_{\mathrm{obs},t}(\theta_A\mid x_t,\theta_c)$, the same belief before the outcome is known,
\begin{equation}
    \label{eq:cavity}
    q_{\mathrm{obs},t}(\theta_A\mid x_t,\theta_c)=\sum_{y}q_{yx,t}(y\mid x_t,\theta_c)\,\mathrm{Dir}(\alpha_A+e_{yx_t\theta_c}),
\end{equation}
a finite Dirichlet mixture, the only non-conjugate object in the construction. We project it onto $\mathrm{Dir}(\alpha_A+\bar{n}_{x_t\theta_c})$ with the expected fractional counts $\bar{n}_{x_t\theta_c}(y):=q_{yx,t}(y\mid x_t,\theta_c)$. It matches the mixture mean, is exact for one-hot $q_{yx,t}(\cdot\mid x_t,\theta_c)$, and is overconfident by $O(1/\alpha_{A,0})$ (Appendix~\ref{app:cavity}). The per-entry gains are then closed-form Dirichlet--Dirichlet divergences,
\begin{equation}
    \label{eq:tilts}
    \Delta_{A,t}(x_t,y_t,\theta_c)=D_{\mathrm{KL}}\big[q_{\mathrm{obs},t}(\theta_A\mid x_t,y_t,\theta_c)\,\big\|\,q_{\mathrm{obs},t}(\theta_A\mid x_t,\theta_c)\big]
\end{equation}
with $\mathcal{I}_{A,t}=\mathbb{E}_{q_{yx,t}}[\Delta_{A,t}]$, and symmetrically $\Delta_{B,t}(x_t,x_{t-1},u_t,\theta_c)$ with cavity $q_{\mathrm{dyn},t}(\theta_B\mid x_{t-1},\theta_c)$ and $\mathcal{I}_{B,t}=\mathbb{E}_{q_{\mathrm{trip},t}}[\Delta_{B,t}]$. Salience needs neither: $\theta_c$ is a stored axis, so $\mathcal{I}_{c,t}=I_{q_{yx,t}}(y_t;\theta_c\mid x_t)$ is a mutual information of the finite array.

\subsection{The Planning Objective in Closed Form}
\label{sec:planning_objective}

Combining the factor beliefs along their separators gives the Bethe joint of the frozen graph, $q=\prod_{t}q_{\mathrm{obs},t}\,q_{\mathrm{dyn},t}\big/\big(q(\theta_c)^{2T-1}\prod_{t<T}q(x_t)^{2}\,q(x_T)\big)$, with exponents from the variable degrees (Appendix~\ref{app:objective_proof}). Its $\theta$-free conditionals are the block conditionals and its $\theta_A,\theta_B$-conditionals the posteriors and cavities of~\eqref{eq:factor_belief}--\eqref{eq:cavity}, so every object in~\eqref{eq:prior_u}--\eqref{eq:prior_xxu} is well-defined under $q$.

\begin{proposition}[Closed-form planning objective]
\label{prop:objective}
Under the Bethe joint $q$: (i)~in the base free energy the parameter entropy of each factor belief cancels per entry against its parameter energy, leaving the Bethe free energy of the discrete graph with predictive energies $-\log m_A$, $-\log m_B$ (eq.~\eqref{eq:base_F}); (ii)~the corrections~\eqref{eq:deltaF} hold unchanged, with $\mathcal{I}_{A,t},\mathcal{I}_{B,t}$ as in~\eqref{eq:tilts} and $\mathcal{I}_{c,t}=I_{q_{yx,t}}(y_t;\theta_c\mid x_t)$; (iii)~hence the planning objective is the closed-form functional
\begin{equation}
    \label{eq:F_T}
    \begin{aligned}
    F_T[q]={}&\textstyle\sum_{t=1}^{T}\Big(D_{\mathrm{KL}}\big[q_{yx,t}\,\big\|\,m_A\,\hat{p}_y\big]+D_{\mathrm{KL}}\big[q_{\mathrm{trip},t}\,\big\|\,m_B\,p(u_t)\,\hat{p}_x\big]\Big)+\mathcal{C}[q]\\
    &+\sum_{t=1}^{T}\Big(\mathbb{H}_{q_{yx,t}}[y_t\mid x_t]-2\,\mathbb{E}_{q_{yx,t}}[\Delta_{A,t}]-2\,I_{q_{yx,t}}(y_t;\theta_c\mid x_t)\\
    &\qquad-\mathbb{H}_{q_{\mathrm{trip},t}}[x_t\mid x_{t-1},u_t]-\mathbb{E}_{q_{\mathrm{trip},t}}[\Delta_{B,t}]+\mathbb{H}_{q_{\mathrm{trip},t}}[u_t\mid x_{t-1}]\Big)
    \end{aligned}
\end{equation}
over the discrete coordinates alone, where $\mathcal{C}[q]$ collects the Bethe counting terms and the $x_0,\theta_c$ priors.
\end{proposition}
\begin{proof}
See Appendix~\ref{app:objective_proof}.
\end{proof}

\subsection{Stationarity Conditions and Message Passing}
\label{sec:channels}

Since $F_T$ has Bethe form over the discrete coordinates, its stationarity conditions under the normalization and consistency multipliers (Appendix~\ref{app:mp_lagrangian}) yield sum-product message updates. Two facts make the variation local and term-by-term (Appendix~\ref{app:mp_stationarity}): entropies of conditionals of a normalized joint differentiate as if the logarithm were frozen (Lemma~\ref{lem:score}), and the dependence of the cavity on $q_{yx,t}$ contributes nothing, $\delta\mathcal{I}_{A,t}/\delta q_{yx,t}(y,x,\theta_c)=\Delta_{A,t}(x,y,\theta_c)$ (Lemma~\ref{lem:cavity}).

\begin{proposition}[Modified factor functions]
\label{prop:kernels}
At stationarity $q^*_{yx,t}\propto\bar{g}_{\mathrm{obs},t}$ and $q^*_{\mathrm{trip},t}\propto\bar{g}_{\mathrm{dyn},t}$, up to the singleton multipliers fixed in Appendix~\ref{app:mp_stationarity}, with the factor functions
\begin{align}
    \bar{g}_{\mathrm{obs},t}(y_t,x_t,\theta_c) &= m_A(y_t\mid x_t,\theta_c)\; \frac{r_t(y_t\mid x_t,\theta_c)^2}{r_t(y_t\mid x_t)}\; e^{2\,\Delta_{A,t}(x_t,y_t,\theta_c)}, \label{eq:kernel_obs}\\
    \bar{g}_{\mathrm{dyn},t}(x_t,x_{t-1},u_t,\theta_c) &= m_B(x_t\mid x_{t-1},u_t,\theta_c)\; \frac{r_t(u_t\mid x_{t-1})}{r_t(x_t\mid x_{t-1},u_t)}\; e^{\Delta_{B,t}}, \label{eq:kernel_dyn}
\end{align}
evaluated at the self-consistent conditionals $r_t(\cdot\mid\cdot)=q^*(\cdot\mid\cdot)$ of the block beliefs.
\end{proposition}
\begin{proof}
See Appendix~\ref{app:mp_stationarity}, which also fixes the multipliers from the singleton stationarity conditions.
\end{proof}

The factor functions~\eqref{eq:kernel_obs}--\eqref{eq:kernel_dyn} are the pair (13a)--(13b) of~\cite{nuijten2026whattypeinference} with the point likelihood replaced by the predictive likelihood and an \emph{information-gain weight} $e^{2\Delta_A}$, $e^{\Delta_B}\ge1$ appended. The weights are not imposed, but emerge from the variation of the information gains.

\begin{algorithm}[t]
\footnotesize
\caption{\textbf{Conjugate planning and learning.} Inputs: counts $\alpha_A,\alpha_B$, beliefs $q(\theta_c),q(x_0)$, horizon $T$, damping $\eta$, tolerance $\epsilon$, sweep budget $N$.}
\label{alg:planner}
\begin{algorithmic}[1]
\Repeat \Comment{one environment step}
    \State $m_A,m_B$ and one-count posteriors $\mathrm{Dir}(\alpha+e)$ from $\alpha_A,\alpha_B$ \Comment{once per call}
    \State $r_t\gets$ uniform, $\bar{n}\gets$ prior predictive, $\mu_{\theta_c\to\cdot}\gets q(\theta_c)$, $\mu_{x_0}\gets q(x_0)$
    \For{$k=1,\dots,N$}
        \State factor functions $\bar{g}_{\mathrm{obs},t},\bar{g}_{\mathrm{dyn},t}$; forward--backward pass \Comment{\eqref{eq:kernel_obs}--\eqref{eq:kernel_dyn}}
        \State refresh $r_t\gets q(\cdot\mid\cdot)$, $\bar{n}$ and the weights $\Delta_{A,t},\Delta_{B,t}$, damped by $\eta$
        \State evaluate $F^{(k)}_T$; \textbf{break if} $|F^{(k)}_T-F^{(k-1)}_T|<\epsilon$ \Comment{\eqref{eq:F_T}}
    \EndFor
    \State execute $u_1\in\arg\max q(u_1)$; observe $y_1$; filter $q(x_1)$; $q(\theta_c)$ by exact Bayes
    \State $\alpha_A\mathrel{+}=\mathbb{E}[e_{y_1x_1\theta_c}]$, $\alpha_B\mathrel{+}=\mathbb{E}[e_{x_1x_0u_1\theta_c}]$ under $q(x_1,x_0)\,q(\theta_c)$; $q(x_0)\gets q(x_1)$
\Until{episode ends}
\end{algorithmic}
\end{algorithm}

\paragraph{Algorithm, learning and the fixed-parameter limit.} With $r_t$ lagged by one sweep, as for the channels of~\cite{nuijten2026whattypeinference}, Proposition~\ref{prop:kernels} is a standard Bethe belief propagation problem on the discrete graph over $\{y_t,x_t,u_t,\theta_c\}$ with factor functions~\eqref{eq:kernel_obs}--\eqref{eq:kernel_dyn}. Algorithm~\ref{alg:planner} solves it with closed-form updates (Appendix~\ref{app:mp_updates}--\ref{app:mp_generic}); the copies $\theta_A^{(t)},\theta_B^{(t)}$ exchange no messages, so the cost is independent of the size of the Dirichlet parameter space. Because $\tilde{p}$ depends on $q$ and the lagged $r_t$ is self-consistent only at convergence, $F_T$ is monitored rather than assumed to descend. The count update is the plain conjugate one: exact stationarity of $F_T$ in $q_{\theta_A}$ adds only a centered score that shifts no count mass (Remark~\ref{rem:app_theta}). As $q_{\theta_A},q_{\theta_B}\to\delta_{\theta^*}$ the weights vanish, $m\to p(\cdot\mid\theta^*,\theta_c)$, and~\eqref{eq:kernel_obs}--\eqref{eq:kernel_dyn} reduce to the fixed-parameter EFE factor functions of~\cite{nuijten2026messagepassingrealizationexpected,nuijten2026whattypeinference} with their enumerable $\theta$ read as $\theta_c$; $K=1$ removes the context (Section~\ref{sec:experiments}).

\section{Experiments}
\label{sec:experiments}

We test whether the information-gain corrections of~\eqref{eq:F_T} produce exploration that a preference-driven planner does not, on two grid POMDPs where knowing the model changes the route and exploration is survivable. Everything outside the objective is shared: the conjugate graph of Section~\ref{sec:methodology} with $K=1$, the exact filter and count update of Algorithm~\ref{alg:planner}, a terminal preference on the goal, and identical horizons, sweeps and flat Dirichlet priors on the unknown block (details in Appendix~\ref{app:experiments}).\footnote{Code: \url{https://github.com/skoghoern/sp-efe-parameter-learning}}

\paragraph{Methods.} \textbf{SP-EFE} is the full scheme with the factor functions~\eqref{eq:kernel_obs}--\eqref{eq:kernel_dyn}. \textbf{SP-Pref-Plan} keeps only the planning correction $r_t(u_t\mid x_{t-1})$ (uniform channels, $\Delta_A=\Delta_B=0$). \textbf{SP-Pref} drops the planning correction too. Both are blind to $m_A$ in planning, since their observation-to-state message is $\sum_y m_A(y\mid x)=1$, and reach an unknown transition only through posterior-mean \emph{optimism}. Four ablations drop terms of~\eqref{eq:F_T} from SP-EFE: $\text{SP-EFE}^{-\mathrm{dyn}}$ (no dynamics channel), $\text{SP-EFE}^{-\Delta_B}$, and $\text{SP-EFE}^{-\Delta_A-\Delta_B}$ with and without the channel (ambiguity, no novelty). \textbf{Random} acts uniformly while learning; \textbf{EFE} enumerates the $4^{H}$ policies under the exact~\eqref{eq:efe}.

\begin{figure}[t]
\centering
\includegraphics[width=0.78\textwidth]{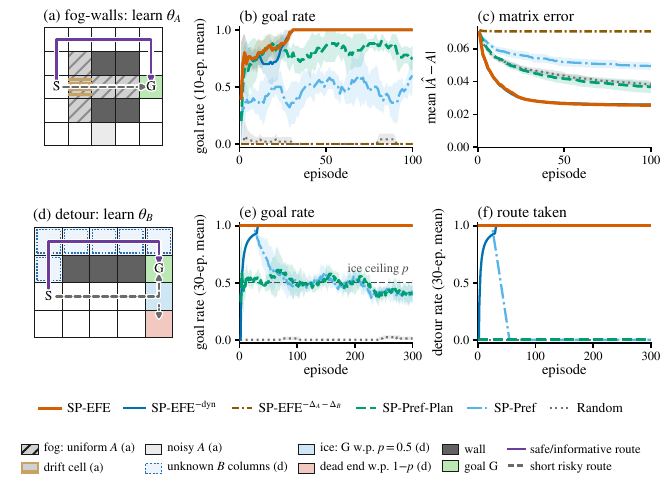}
\caption{\textbf{Top (a--c): learning $\theta_A$} on fog-walls ($B$ known, $H=8$, 5 seeds); \textbf{bottom (d--f): learning $\theta_B$} on the unmapped detour (identity $A$, $H=7$, $T=8$, 3 seeds). Curves are trailing means with a $\pm1$~std band over seeds; the glyph legend gives the cell and route vocabulary of (a) and (d).}
\label{fig:learning}
\end{figure}

\paragraph{Learning the observation model (fog-walls).} In the fog-walls maze (Figure~\ref{fig:learning}a) the direct 4-step route crosses a drift cell and an ambiguous fog block, whereas the 8-step top row identifies every state; walls make mislocalization cost bumps rather than the episode. The top row pays off only for a planner that \emph{knows} it is unambiguous: with $A,B$ known, SP-EFE takes it on 30/30 seeds, whereas SP-Pref-Plan, whose route is $A$-independent, enters the fog and scores 18/30 (Appendix~\ref{app:exp_controls}). When learning $A$ (Figure~\ref{fig:learning}b--c, Table~\ref{tab:results}), SP-EFE reaches a goal rate of $1.00$ from episode ${\approx}30$ on every seed, while SP-Pref-Plan plateaus at $0.81$ and SP-Pref at $0.48$ with $1.4$--$1.9\times$ the matrix error: by construction, learning cannot change their route. The curve of SP-EFE is a novelty phase followed by a switch: under the flat prior every cell is equally novel, so the agent enters the fog. Once $\mathbb{H}[m_A(\cdot\mid x)]$ is maximal in the fog and near zero elsewhere, ambiguity reroutes the plan along a deterministic row and the diminished $\Delta_A$ no longer pulls it back.

\paragraph{Learning the transition model (unmapped detour).} The unmapped detour (Figure~\ref{fig:learning}d) caps a planner that never learns at $p=0.5$, and optimism is priced out by design: a flat column spreads the successor over all $N=20$ cells, which reaches the goal within the horizon with probability $0.265<p$, so under its own objective SP-Pref-Plan \emph{should} gamble. Table~\ref{tab:results} reports three scenarios. With $B$ known, every SP-EFE variant and SP-Pref-Plan take the detour on 300/300 episodes, so learning is \emph{sufficient}; SP-Pref gambles even with a perfect model (Appendix~\ref{app:exp_controls}). With $B$ frozen at the prior, SP-Pref-Plan sits at $146\approx p$ and never opens a door, so learning is \emph{necessary}. When learning $B$, SP-EFE takes the detour from episode~1 and never leaves it, whereas SP-Pref-Plan gambles on all 300 episodes and $m_B$ ends equal to its prior. SP-Pref walks the detour only while the smeared doors still pay its action-entropy bonus, then gambles for the remaining 275 episodes: a better model removes the only term that steered it to the safe route. Enumerated EFE takes the detour on 100/100 episodes, so the exploration is a property of the objective, not of the message passing.

\begin{table}[t]
\centering
\caption{\textbf{Results and ablation.} Fog-walls, learning $\theta_A$ (100 episodes, 5 seeds): goals, goals in the last 20, final matrix error. Unmapped detour, learning $\theta_B$ (300 episodes, 3 seeds): goals with $B$ known, frozen at the prior, or learned, and learned-$B$ goals via the detour. Lower block: ablations of~\eqref{eq:F_T}. $^\dagger$Entropy-following: with frozen flat doors, the most uncertain column and the detour coincide.}
\label{tab:results}
\scriptsize
\setlength{\tabcolsep}{3.2pt}
\renewcommand{\arraystretch}{1.0}
\begin{tabular}{@{}l ccc c cccc@{}}
\toprule
& \multicolumn{3}{c}{fog-walls, learn $\theta_A$} & & \multicolumn{4}{c}{unmapped detour, learn $\theta_B$} \\
\cmidrule{2-4}\cmidrule{6-9}
Planner & goals/100 & last 20 & err$_A$ & & known & frozen & learned & detour \\
\midrule
SP-EFE & $\mathbf{95.8\pm0.7}$ & \textbf{20.0} & \textbf{0.026} & & \textbf{300} & $300^{\dagger}$ & \textbf{300} & \textbf{300} \\
SP-Pref-Plan & $81.4\pm4.8$ & 16.2 & 0.037 & & 300 & 146 & $146\pm7$ & 0 \\
SP-Pref & $47.4\pm3.1$ & 9.6 & 0.049 & & 146 & $300^{\dagger}$ & $158\pm8$ & 25 \\
Random & $1.0\pm0.9$ & 0.2 & 0.038 & & 1 & 1 & $1.3\pm0.5$ & 0 \\
\midrule
$\text{SP-EFE}^{-\mathrm{dyn}}$ & $94.4\pm2.0$ & 20.0 & 0.026 & & 300 & 0 & 298 & 298 \\
$\text{SP-EFE}^{-\Delta_B}$ & $94.4\pm2.2$ & 20.0 & 0.025 & & 300 & $300^{\dagger}$ & 300 & 300 \\
$\text{SP-EFE}^{-\Delta_A-\Delta_B}$ & $0.0\pm0.0$ & 0.0 & 0.071 & & 300 & $300^{\dagger}$ & 300 & 300 \\
$\text{SP-EFE}^{-\Delta_A-\Delta_B-\mathrm{dyn}}$ & $0.0\pm0.0$ & 0.0 & 0.071 & & 300 & 146 & $146\pm7$ & 0 \\
\bottomrule
\end{tabular}
\end{table}

\paragraph{Ablation.} The lower block of Table~\ref{tab:results} removes terms of~\eqref{eq:F_T}. \emph{Novelty is the drive.} Without $\Delta_A$, the channels-only planner keeps the ambiguity penalty but has nothing that rewards resolving it; from a flat prior every unvisited cell is maximally ambiguous ($\log25$ nats per step against a preference of $4$), and the agent bumps into the wall beside the start on all 500 episodes. Without $\Delta_B$, the learned-$B$ score is carried by the dynamics channel, which the frozen column exposes as entropy-following. \emph{The dynamics channel is dispensable.} Removing it changes nothing on fog-walls and costs two episodes on the detour; at its fixed point the channel cancels $m_B$ on its support, so a smeared learned $m_B$ loses the exploitation signal.

\section{Conclusion}
\label{sec:conclusion}

This paper extended entropy-corrected EFE-based planning~\cite{nuijten2026whattypeinference} from inference over a finite set of known models to conjugate \emph{learning} of the observation and transition matrices, as the minimization of a single variational objective. Frozen per-factor priors keep every factor belief conjugate and preserve the parameter--datum joint that mean-field destroys. Thus the parameter side collapses into predictive likelihoods, while the information gains survive as weights on the factor functions. On two grid POMDPs those weights are the drive: the scheme learns $\theta_A$ or $\theta_B$ and settles on the safe route where SP-MP with the same inference update plateaus or gambles.

\paragraph{Limitations and future work.} The empirical advantage has conditions, all visible in the designs of Section~\ref{sec:experiments}: exploration must be survivable, the unknown region small enough for novelty to satiate, and model knowledge must change the route, since greedy optimism already wins on a unique-path task with a hard deadline. The objective is non-convex and the factor functions are built from lagged factors, so damped descent need not be monotone and $F_T$ measures deviations without certifying convergence. Future work includes score-corrected learning, and sophisticated planning beyond the frozen parameter priors.

\begin{credits}
\subsubsection{\discintname}
The authors have no competing interests to declare that are relevant to the content of this article.
\end{credits}

\bibliographystyle{splncs04}
\bibliography{mybibliography}
\clearpage
\appendix
\renewcommand{\theHsection}{appendix.\thesection}
\renewcommand{\theHsubsection}{appendix.\thesubsection}
\section{The Objective: Proofs and Constructions}
\label{app:priors}

Notation: $\theta=(\theta_A,\theta_B,\theta_c)$, $\boldsymbol{z}=(\boldsymbol{y},\boldsymbol{x},\boldsymbol{u},\theta)$, $\mathbb{E}_q=\mathbb{E}_{q(\boldsymbol{z})}$; $h[q(\cdot\mid x)]$ is the entropy of a conditional as a function of $x$ and $\mathbb{H}[q(\cdot\mid x)]=\mathbb{E}_{q(x)}h[q(\cdot\mid x)]$ the conditional-entropy scalar; $\mathrm{Dir}(\alpha)$ with $\alpha_0=\sum_k\alpha_k$ has density $\prod_k\theta_k^{\alpha_k-1}/\mathrm{B}(\alpha)$. Time indices are suppressed where no confusion arises.

\subsection{Epistemic Priors Recover the EFE}
\label{app:priors_efe}

\begin{lemma}[Prior contributions]
\label{lem:app_priors}
For the priors~\eqref{eq:prior_u}--\eqref{eq:prior_xxu} and any joint $q$ possessing the conditionals they refer to,
\begin{align}
    -\mathbb{E}_q[\log\tilde{p}(u_t)] &= -\mathbb{H}[q(x_t\mid x_{t-1},u_t)], \label{eq:app_lemma_u}\\
    -\mathbb{E}_q[\log\tilde{p}(x_t)] &= +\mathbb{H}[q(y_t\mid x_t,\theta_A,\theta_c)], \label{eq:app_lemma_x}\\
    -\mathbb{E}_q[\log\tilde{p}(y_t,x_t)] &= -I(\theta_A,\theta_c;y_t\mid x_t) = -\mathcal{I}_{c,t}-\mathcal{I}_{A,t}, \label{eq:app_lemma_yx}\\
    -\mathbb{E}_q[\log\tilde{p}(x_t,x_{t-1},u_t,\theta_c)] &= -I(\theta_B;x_t,u_t\mid x_{t-1},\theta_c)=-\mathcal{I}_{B,t}, \label{eq:app_lemma_xxu}
\end{align}
up to the normalizing constants of the priors.
\end{lemma}
\begin{proof}
Each prior is the exponential of a functional of the variables it is indexed by; taking $-\log$ and the outer expectation over those variables gives the stated quantity by definition. For~\eqref{eq:app_lemma_u}, $\mathbb{E}_{q(u_t)}\mathbb{E}_{q(x_t,x_{t-1}\mid u_t)}[\log q(x_t\mid x_{t-1},u_t)]=-\mathbb{H}[q(x_t\mid x_{t-1},u_t)]$; for~\eqref{eq:app_lemma_x}, $\mathbb{E}_{q(x_t)}\mathbb{E}_{q(\theta_A,\theta_c\mid x_t)}h[q(y_t\mid x_t,\theta_A,\theta_c)]=\mathbb{H}[q(y_t\mid x_t,\theta_A,\theta_c)]$; for~\eqref{eq:app_lemma_yx}, $\mathbb{E}_{q(y_t,x_t)}D_{\mathrm{KL}}[q(\theta_A,\theta_c\mid y_t,x_t)\,\|\,q(\theta_A,\theta_c\mid x_t)]$ is the conditional mutual information $I(\theta_A,\theta_c;y_t\mid x_t)$ written as an expected divergence between posterior and cavity, and the chain rule $I(\theta_A,\theta_c;y_t\mid x_t)=I(\theta_c;y_t\mid x_t)+I(\theta_A;y_t\mid x_t,\theta_c)$ gives the split into salience and novelty of Section~\ref{sec:generative_model}; likewise~\eqref{eq:app_lemma_xxu} with $(x_t,u_t)$ as the datum and $(x_{t-1},\theta_c)$ as the conditioning variables. \qed
\end{proof}

\begin{proof}[of Corollary~\ref{cor:dual_priors}]
Split $\log\tilde{p}$ into the unbiased model, the goals and the epistemic priors,
\begin{equation}
    \log\tilde{p}=\log p+\sum_{t=1}^{T}\log\hat{p}_x(x_t)\,\hat{p}_y(y_t)+\sum_{t=1}^{T}\log\big[\tilde{p}(u_t)\,\tilde{p}(x_t)\,\tilde{p}(y_t,x_t)\,\tilde{p}(x_t,x_{t-1},u_t,\theta_c)\big].
\end{equation}
By Lemma~\ref{lem:app_priors},
\begin{equation}
    \begin{aligned}
    F_{\tilde{p}}[q]=\underbrace{\mathbb{E}_q\Big[\log\frac{q}{p}\Big]}_{B[q]}
    +\sum_{t=1}^{T}\Big(&-\mathbb{E}_q\big[\log\hat{p}_x(x_t)\hat{p}_y(y_t)\big]-\mathbb{H}[q(x_t\mid x_{t-1},u_t)]\\
    &+\mathbb{H}[q(y_t\mid x_t,\theta_A,\theta_c)]-\mathcal{I}_{c,t}-\mathcal{I}_{A,t}-\mathcal{I}_{B,t}\Big).
    \end{aligned}
\end{equation}
Expanding~\eqref{eq:efe} and taking $\mathbb{E}_{q(u_t)}$,
\begin{equation}
    \begin{aligned}
    \mathbb{E}_{q(u_t)}[G(u_t)]=&-\mathbb{E}_q\big[\log\hat{p}_x(x_t)\hat{p}_y(y_t)\big]-\mathbb{H}[q(x_t\mid x_{t-1},u_t)]\\
    &+\mathbb{E}_q[-\log p(y_t\mid x_t,\theta_A,\theta_c)]-\mathcal{I}_{c,t}-\mathcal{I}_{A,t}-\mathcal{I}_{B,t},
    \end{aligned}
\end{equation}
which coincides with the $t$-th summand above whenever $q(y_t\mid x_t,\theta_A,\theta_c)=p(y_t\mid x_t,\theta_A,\theta_c)$. For the factor beliefs of Section~\ref{sec:family}, $q_{\mathrm{obs},t}(y_t\mid x_t,\theta_A,\theta_c)\propto p(y_t\mid x_t,\theta_A,\theta_c)\,\hat{p}_y(y_t)$, so the identity is exact for uniform $\hat{p}_y$ (the setting of Section~\ref{sec:experiments}); in general the ambiguity term of $G$ is to be read as $\mathbb{H}[q(y_t\mid x_t,\theta_A,\theta_c)]$, the entropy of the agent's own predictive observation channel, as in~\cite{nuijten2026whattypeinference}. Summing over $t$ gives the Corollary. \qed
\end{proof}

\subsection{Entropy Corrections}
\label{app:priors_entropy}

Adding the planning correction $+\mathbb{H}[q(u_t\mid x_{t-1})]$~\cite{lázarogredilla2025typeinferenceplanning} to the sum of~\eqref{eq:app_lemma_u}--\eqref{eq:app_lemma_xxu} gives the raw per-step correction
\begin{equation}
    \label{eq:app_deltaF_raw}
    \Delta F_t=\mathbb{H}[q(y_t\mid x_t,\theta_A,\theta_c)]-\mathcal{I}_{c,t}-\mathcal{I}_{A,t}-\mathbb{H}[q(x_t\mid x_{t-1},u_t)]-\mathcal{I}_{B,t}+\mathbb{H}[q(u_t\mid x_{t-1})].
\end{equation}
The identity $\mathbb{H}[q(y_t\mid x_t,\theta_A,\theta_c)]=\mathbb{H}[q(y_t\mid x_t)]-I(\theta_A,\theta_c;y_t\mid x_t)=\mathbb{H}[q(y_t\mid x_t)]-\mathcal{I}_{c,t}-\mathcal{I}_{A,t}$ turns~\eqref{eq:app_deltaF_raw} into~\eqref{eq:deltaF}: the ambiguity prior contributes one unit of each observation-side gain through its $\theta$-coupling, the novelty prior a second. Under any family with $q(y_t,x_t,\theta_c,\theta_A)=q(y_t,x_t,\theta_c)\,q(\theta_A)$ the Dirichlet gain $\mathcal{I}_{A,t}$ vanishes identically and prior~\eqref{eq:prior_yx} retains only its salience part; the same holds for $\mathcal{I}_{B,t}$ and~\eqref{eq:prior_xxu}.

\subsection{The Conjugate Factor Belief}
\label{app:family}

\begin{lemma}[Entrywise conjugacy]
\label{lem:app_conjugacy}
For one column $\alpha=\alpha_A(\cdot\mid x,\theta_c)$ and any entry $y$,
$\theta_y\,\mathrm{Dir}(\theta\mid\alpha)=\dfrac{\alpha_y}{\alpha_0}\,\mathrm{Dir}(\theta\mid\alpha+e_y)$.
\end{lemma}
\begin{proof}
$\theta_y\prod_k\theta_k^{\alpha_k-1}/\mathrm{B}(\alpha)=\prod_k\theta_k^{\alpha_k+[k=y]-1}/\mathrm{B}(\alpha)=\dfrac{\mathrm{B}(\alpha+e_y)}{\mathrm{B}(\alpha)}\,\mathrm{Dir}(\theta\mid\alpha+e_y)$ and $\dfrac{\mathrm{B}(\alpha+e_y)}{\mathrm{B}(\alpha)}=\dfrac{\Gamma(\alpha_y+1)\,\Gamma(\alpha_0)}{\Gamma(\alpha_y)\,\Gamma(\alpha_0+1)}=\dfrac{\alpha_y}{\alpha_0}$. Other columns are untouched. \qed
\end{proof}

Applied, for each value of $\theta_c$, to the factor belief of Section~\ref{sec:family},
\begin{equation*}
    q_{\mathrm{obs},t}(y_t,x_t,\theta_c,\theta_A^{(t)})\propto p(y_t\mid x_t,\theta_A^{(t)},\theta_c)\,\mu_{y_t}(y_t)\,\mu_{x_t}(x_t)\,\mu_{\theta_c}(\theta_c)\,\mathrm{Dir}(\theta_A^{(t)}\mid\alpha_A),
\end{equation*}
this gives~\eqref{eq:factor_belief}. The identity needs two things: the message into the copy must be free of $(y,x,\theta_c)$---true because the copy has degree two and its other neighbor is the prior node---and the factor function must be a monomial in $\theta_A$ for each value of $\theta_c$. A $\theta_A$-conditioned channel inside the factor function, such as $r(y\mid x,\theta_A,\theta_c)^2$ with the shared parameter of~\cite{nuijten2026whattypeinference} read as $\theta_A$, has a normalizer that is not polynomial in $\theta_A$ and would destroy the identity; this is why the $\theta_A$-epistemic content enters the factor function as a scalar weight per discrete entry, while the finite $\theta_c$ retains the channel form (Proposition~\ref{prop:kernels}).

\begin{lemma}[Elimination of the $\theta_A$-conditional]
\label{lem:app_elimination}
Let the observation region carry the belief $q_{yx,t}(y,x,\theta_c)\,\nu_t(\theta_A\mid y,x,\theta_c)$ for an arbitrary normalized $\nu_t$, and let the copy's prior region carry the induced marginal $\sum_{y,x,\theta_c}q_{yx,t}\nu_t$. Their joint contribution to the Bethe free energy of the frozen graph is
\begin{equation}
    \label{eq:app_elimination}
    -\mathbb{H}[q_{yx,t}]-\mathbb{E}_{q_{yx,t}}[\log m_A]+\sum_{y,x,\theta_c}q_{yx,t}(y,x,\theta_c)\,D_{\mathrm{KL}}\big[\nu_t(\cdot\mid y,x,\theta_c)\,\big\|\,\mathrm{Dir}(\alpha_A+e_{yx\theta_c})\big],
\end{equation}
minimized uniquely at $\nu_t=\mathrm{Dir}(\alpha_A+e_{yx\theta_c})=q_{\mathrm{obs},t}(\theta_A\mid x,y,\theta_c)$.
\end{lemma}
\begin{proof}
The copy has degree two, so its region term $D_{\mathrm{KL}}[q(\theta_A^{(t)})\,\|\,\mathrm{Dir}(\alpha_A)]$ and its counting term $+\mathbb{H}[q(\theta_A^{(t)})]$ combine to $-\mathbb{E}_{q(\theta_A^{(t)})}[\log\mathrm{Dir}(\theta_A\mid\alpha_A)]$. Adding the observation region $D_{\mathrm{KL}}[q_{yx,t}\nu_t\,\|\,p(y\mid x,\theta_A,\theta_c)]$ and using $q(\theta_A^{(t)})=\sum_{y,x,\theta_c}q_{yx,t}\nu_t$,
\begin{equation*}
    -\mathbb{H}[q_{yx,t}]+\sum_{y,x,\theta_c}q_{yx,t}\,\mathbb{E}_{\nu_t}\Big[\log\frac{\nu_t(\theta_A\mid y,x,\theta_c)}{\mathrm{Dir}(\theta_A\mid\alpha_A)\,p(y\mid x,\theta_A,\theta_c)}\Big],
\end{equation*}
and Lemma~\ref{lem:app_conjugacy} rewrites the denominator as $m_A(y\mid x,\theta_c)\,\mathrm{Dir}(\theta_A\mid\alpha_A+e_{yx\theta_c})$. \qed
\end{proof}

A coordinate whose optimum is a closed-form function of the others is substituted out; it is also not storable, $\theta_A$ being continuous. Two consequences. First, the parameter energy and entropy of the region collapse jointly into the \emph{predictive} energy $-\log m_A$; nothing $\theta_A$-dependent survives in the base free energy (used in Appendix~\ref{app:objective_proof}). Second, the corrections~\eqref{eq:deltaF} are \emph{evaluated} at $\nu_t=\mathrm{Dir}(\alpha_A+e_{yx\theta_c})$ and not re-minimized over $\nu_t$: read as a functional of a free $\nu_t$, $-2\mathcal{I}_{A,t}$ would reward arbitrary coupling up to its ceiling $\mathbb{H}_q[y_t\mid x_t,\theta_c]$. The dynamics side is identical with columns indexed by $(x_{t-1},u_t,\theta_c)$. The finite $\theta_c$ is treated differently: a free $\theta_c$-axis on the discrete beliefs costs a factor $K$ and is kept, so its conditionals are exact and no projection is needed.

The step-wise copies are what make this consistent. A single shared $\theta_A$ would impose the calibration constraint $\sum_{y,x,\theta_c}q_{\mathrm{obs},t}=q(\theta_A)$ for every $t$, but $\sum_{y,x,\theta_c}q_{\mathrm{obs},t}=\mathrm{Dir}(\alpha_A)\sum_{y,x,\theta_c}q_{yx,t}(y,x,\theta_c)\,\theta_A(y\mid x,\theta_c)/m_A(y\mid x,\theta_c)$ equals $\mathrm{Dir}(\alpha_A)$ only when $q_{yx,t}(y\mid x,\theta_c)=m_A(y\mid x,\theta_c)$, i.e.\ when planning leaves the block conditionals at their predictive values; enforcing it would write $\sim T$ imagined counts into $\alpha_A$ per sweep and eliminate novelty inside the planning loop. With step-wise copies the constraint holds trivially and the approximation moves out of the constraint set into the planning-time model.

\subsection{The Cavity and Its Projection}
\label{app:cavity}

Fix a column $(x,\theta_c)$, write $r_y:=q_{yx,t}(y\mid x,\theta_c)$, $p_y:=\mathrm{Dir}(\alpha+e_y)$, and let $\bar{q}:=\sum_y r_y\,p_y$ be the exact cavity~\eqref{eq:cavity}. The projection $\tilde{q}:=\mathrm{Dir}(\alpha+r)$ has the following properties.
\begin{enumerate}
    \item[(i)] \emph{Mean matching.} $\mathbb{E}_{\bar{q}}[\theta_k]=\sum_y r_y\frac{\alpha_k+[k=y]}{\alpha_0+1}=\frac{\alpha_k+r_k}{\alpha_0+1}=\mathbb{E}_{\tilde{q}}[\theta_k]$.
    \item[(ii)] \emph{Geometric mixture.} $\prod_y p_y^{\,r_y}\propto\prod_k\theta_k^{\sum_y r_y(\alpha_k+[k=y]-1)}=\prod_k\theta_k^{\alpha_k+r_k-1}\propto\tilde{q}$: the projection averages the branches' natural parameters. Equivalently $\tilde{q}=\arg\min_{\tilde{q}'}\sum_y r_y D_{\mathrm{KL}}[\tilde{q}'\,\|\,p_y]$, whereas $\bar{q}=\arg\min_{\tilde{q}'}\sum_y r_y D_{\mathrm{KL}}[p_y\,\|\,\tilde{q}']$.
    \item[(iii)] \emph{Structured VMP.} $\tilde{q}$ is what variational message passing returns on the $\theta$-edge under $q(y,\theta\mid x)=q(y\mid x)\,q(\theta\mid x)$: prior times $\exp\mathbb{E}_{q(y\mid x)}[\log\theta_y]=\prod_y\theta_y^{r_y}$~\cite{senoz2021variational}; the sum-product message $\sum_y r_y\theta_y$ gives the arithmetic mixture $\bar{q}$ instead.
    \item[(iv)] \emph{Not the M-projection.} $\arg\min_{\tilde{q}'\in\mathrm{Dir}}D_{\mathrm{KL}}[\bar{q}\,\|\,\tilde{q}']$ would match $\mathbb{E}[\log\theta_k]$. The mixture carries between-branch variance that a single Dirichlet drops,
    $\mathbb{E}_{\bar{q}}[\theta_k^2]-\mathbb{E}_{\tilde{q}}[\theta_k^2]=\frac{r_k(1-r_k)}{(\alpha_0+1)(\alpha_0+2)}$,
    so $\tilde{q}$ is slightly overconfident and $\Delta_{A,t}$ slightly overstated, with relative error $O(1/\alpha_0)$ that vanishes as counts grow.
    \item[(v)] \emph{Exactness and closure.} $\tilde{q}=\bar{q}$ whenever $r$ is one-hot; $\tilde{q}$ stays Dirichlet under linear damping of $r$; and $D_{\mathrm{KL}}[\mathrm{Dir}\,\|\,\mathrm{Dir}]$ is closed-form in digamma and log-gamma functions.
\end{enumerate}
The projection is used for the weight arrays only; the discrete messages close through $m_A,m_B$ without it. Its one side effect on the stationarity conditions is quantified in Lemma~\ref{lem:cavity}.

\subsection{Proof of Proposition~\ref{prop:objective}}
\label{app:objective_proof}

\paragraph{Degrees and Bethe free energy.} On the frozen graph, $y_t$, $u_t$, $x_0$ and every copy $\theta_A^{(t)},\theta_B^{(t)}$ have one non-singleton neighbor besides their prior or goal node; $x_t$ for $1\le t<T$ has three ($f_{\mathrm{obs},t},f_{\mathrm{dyn},t},f_{\mathrm{dyn},t+1}$), $x_T$ two, and $\theta_c$ all $2T$. For a variable with $k$ non-singleton neighbors and singleton factor $f_i$, the singleton region and the counting term combine to $(k-1)\,\mathbb{H}[q_i]-\mathbb{E}_{q_i}[\log f_i]$. Hence
\begin{equation}
    \label{eq:app_bethe_frozen}
    \begin{aligned}
    F_{\hat{p}}[q]=\sum_{t=1}^{T}\Big(&D_{\mathrm{KL}}\big[q_{\mathrm{obs},t}\,\|\,p(y_t\mid x_t,\theta_A^{(t)},\theta_c)\big]-\mathbb{E}_{q(\theta_A^{(t)})}\big[\log\mathrm{Dir}(\theta_A^{(t)}\mid\alpha_A)\big]\\
    &+D_{\mathrm{KL}}\big[q_{\mathrm{dyn},t}\,\|\,p(x_t\mid x_{t-1},u_t,\theta_B^{(t)},\theta_c)\big]-\mathbb{E}_{q(\theta_B^{(t)})}\big[\log\mathrm{Dir}(\theta_B^{(t)}\mid\alpha_B)\big]\\
    &-\mathbb{E}_q\big[\log p(u_t)\,\hat{p}_y(y_t)\,\hat{p}_x(x_t)\big]\Big)\\
    +2\sum_{t=1}^{T-1}&\mathbb{H}[q(x_t)]+\mathbb{H}[q(x_T)]+(2T-1)\,\mathbb{H}[q(\theta_c)]-\mathbb{E}_q[\log p(x_0)\,p(\theta_c)],
    \end{aligned}
\end{equation}
and the same degree count gives the exponents of the Bethe joint of Section~\ref{sec:planning_objective}. Under the planning convention that $p(\theta_c)$ is the current context belief, the $\theta_c$ terms are constants during planning.

\paragraph{(i) Base.} Substitute~\eqref{eq:factor_belief}, i.e.\ $\nu_t=\mathrm{Dir}(\alpha_A+e_{yx\theta_c})$, into Lemma~\ref{lem:app_elimination}: the divergence term vanishes and the first two terms of~\eqref{eq:app_bethe_frozen} equal $-\mathbb{H}[q_{yx,t}]-\mathbb{E}_{q_{yx,t}}[\log m_A]$; likewise on the dynamics side with $m_B$. Hence
\begin{equation}
    \label{eq:base_F}
    \begin{aligned}
    F_{\hat{p}}[q] = \sum_{t=1}^{T}\Big(
    &-\mathbb{H}[q_{yx,t}] - \mathbb{E}_{q_{yx,t}}[\log m_A]
    -\mathbb{H}[q_{\mathrm{trip},t}]\\
    &- \mathbb{E}_{q_{\mathrm{trip},t}}[\log m_B]
    -\mathbb{E}_q\big[\log p(u_t)\,\hat{p}_y(y_t)\,\hat{p}_x(x_t)\big]\Big)-\mathbb{E}_q[\log p(x_0)\,p(\theta_c)]\\
    &+2\textstyle\sum_{t=1}^{T-1}\mathbb{H}[q(x_t)]+\mathbb{H}[q(x_T)]+(2T-1)\,\mathbb{H}[q(\theta_c)].
    \end{aligned}
\end{equation}
Per entry the cancellation reads
\begin{equation}
    \label{eq:app_cancellation}
    \begin{aligned}
    -h\big[\mathrm{Dir}(\alpha_A+e_{yx\theta_c})\big]-\mathbb{E}_{\mathrm{Dir}(\alpha_A+e_{yx\theta_c})}\big[\log p(y\mid x,\theta_A,\theta_c)&+\log\mathrm{Dir}(\theta_A\mid\alpha_A)\big]\\
    &=-\log m_A(y\mid x,\theta_c),
    \end{aligned}
\end{equation}
which is Bayes' rule $\log\mathrm{Dir}(\alpha_A+e_{yx\theta_c})=\log\mathrm{Dir}(\alpha_A)+\log p-\log m_A$ integrated against itself. Because the copy's prior node carries the current belief $\mathrm{Dir}(\alpha_A)$, there is no prior-mismatch residual $\mathbb{E}[\log q(\theta_A)/p(\theta_A)]$ to dispose of.

\paragraph{(ii) Corrections.} Each copy lies in exactly one region of the Bethe joint, so the marginal of $q$ over $(y_t,x_t,\theta_c,\theta_A^{(t)})$ is $q_{\mathrm{obs},t}$ and over $(x_t,x_{t-1},u_t,\theta_c,\theta_B^{(t)})$ is $q_{\mathrm{dyn},t}$; on the chain the Bethe joint is a proper distribution with these marginals. Its $\theta_A,\theta_B$-free conditionals are the block conditionals, $q(y_t\mid x_t)=q_{yx,t}(y_t\mid x_t)$, $q(\theta_c\mid x_t,y_t)=q_{yx,t}(\theta_c\mid x_t,y_t)$, $q(x_t\mid x_{t-1},u_t)=q_{\mathrm{trip},t}(x_t\mid x_{t-1},u_t)$, $q(u_t\mid x_{t-1})=q_{\mathrm{trip},t}(u_t\mid x_{t-1})$; its $\theta_A$-conditionals are $q_{\mathrm{obs},t}(\theta_A\mid x_t,y_t,\theta_c)=\mathrm{Dir}(\alpha_A+e_{y_tx_t\theta_c})$ and the cavity~\eqref{eq:cavity}. Therefore $I_q(\theta_c;y_t\mid x_t)=I_{q_{yx,t}}(y_t;\theta_c\mid x_t)$, $I_q(\theta_A^{(t)};y_t\mid x_t,\theta_c)=\mathbb{E}_{q_{yx,t}}D_{\mathrm{KL}}[q_{\mathrm{obs},t}(\theta_A\mid x_t,y_t,\theta_c)\,\|\,q_{\mathrm{obs},t}(\theta_A\mid x_t,\theta_c)]=\mathbb{E}_{q_{yx,t}}[\Delta_{A,t}]$, likewise for $\theta_B$, and~\eqref{eq:deltaF} holds with every term read off $q_{yx,t}$ and $q_{\mathrm{trip},t}$. Nothing is double-counted with (i): the information-gain mass in the $\theta$-coupled energy is cancelled by~\eqref{eq:app_cancellation} and never reaches $\Delta F_t$.

\paragraph{(iii) Assembly.} Adding (i) and (ii) gives~\eqref{eq:F_T}, a functional of the discrete coordinates of Section~\ref{sec:family} and the fixed counts $\alpha_A,\alpha_B$ only. \qed

\begin{remark}[Limits]
\label{rem:app_limits}
At $q(\theta_A),q(\theta_B)=\delta_{\theta^*}$ every Dirichlet posterior and cavity is $\delta_{\theta^*}$, $\Delta_A=\Delta_B=0$, $m\to p(\cdot\mid\theta^*,\theta_c)$, and~\eqref{eq:F_T} is the fixed-parameter EFE-planning objective of~\cite{nuijten2026whattypeinference} with their $\theta$ read as $\theta_c$. Deleting the correction line leaves belief propagation on the posterior predictive likelihoods. Both limits are approached continuously in the concentrations.
\end{remark}

\section{Message-Passing Derivation}
\label{app:mp}

We derive the stationarity conditions of the planning objective~\eqref{eq:F_T} and solve them as message passing, following Appendix~D of~\cite{nuijten2026whattypeinference} with two structural differences: the Dirichlet copies are integrated out in closed form, so there are no $\theta_A,\theta_B$-messages, and the channels are conditionals of $q$, so there are no channel variables or constraints. The context $\theta_c$ is a discrete variable of the graph and is treated exactly as the shared $\theta$ of~\cite{nuijten2026whattypeinference}.

\subsection{Lagrangian}
\label{app:mp_lagrangian}

At $T{=}1$ the coordinates are $q_{yx}(y_1,x_1,\theta_c)$, $q_{\mathrm{trip}}(x_1,x_0,u_1,\theta_c)$ and the singletons $q_{x_0},q_{x_1},q_{y_1},q_{u_1},q_{\theta_c}$; the parameter beliefs $\mathrm{Dir}(\alpha_A),\mathrm{Dir}(\alpha_B)$ are fixed. With multipliers for normalization of the two blocks and for the seven factor-to-singleton consistency constraints,
\begin{equation}
    \label{eq:app_lagrangian}
    \begin{aligned}
    \mathcal{L}[q,\lambda]=F_1[q]
    &+\lambda_{yx}\Big(\textstyle\sum q_{yx}-1\Big)+\lambda_{\mathrm{trip}}\Big(\textstyle\sum q_{\mathrm{trip}}-1\Big)\\
    &+\textstyle\sum_{y_1}\lambda_{y_1}\Big(\sum_{x_1,\theta_c}q_{yx}-q_{y_1}\Big)
    +\sum_{x_1}\lambda^{\mathrm{obs}}_{x_1}\Big(\sum_{y_1,\theta_c}q_{yx}-q_{x_1}\Big)\\
    &+\textstyle\sum_{\theta_c}\lambda^{\mathrm{obs}}_{\theta_c}\Big(\sum_{y_1,x_1}q_{yx}-q_{\theta_c}\Big)
    +\sum_{x_1}\lambda^{\mathrm{dyn}}_{x_1}\Big(\sum_{x_0,u_1,\theta_c}q_{\mathrm{trip}}-q_{x_1}\Big)\\
    &+\textstyle\sum_{x_0}\lambda_{x_0}\Big(\sum_{x_1,u_1,\theta_c}q_{\mathrm{trip}}-q_{x_0}\Big)
    +\sum_{u_1}\lambda_{u_1}\Big(\sum_{x_1,x_0,\theta_c}q_{\mathrm{trip}}-q_{u_1}\Big)\\
    &+\textstyle\sum_{\theta_c}\lambda^{\mathrm{dyn}}_{\theta_c}\Big(\sum_{x_1,x_0,u_1}q_{\mathrm{trip}}-q_{\theta_c}\Big),
    \end{aligned}
\end{equation}
where each consistency multiplier is a function of its variable, $\lambda_{y_1}=\lambda_{y_1}(y_1)$ etc. Relative to~\cite{nuijten2026whattypeinference} (their (D.3)--(D.7)), the four channel multipliers are absent and the $\theta$-consistency multipliers refer to $\theta_c$ only. The corrections in $F_1$ are functionals of the coordinates through~\eqref{eq:factor_belief}--\eqref{eq:cavity}, cavities included.

\subsection{Stationarity Conditions}
\label{app:mp_stationarity}

Two lemmas make the variation of the corrections local.

\begin{lemma}[Score lemma]
\label{lem:score}
For a normalized joint $B[q](a,b)$ depending on a coordinate $q$, and any of its conditionals,
\begin{equation}
    \frac{\delta}{\delta q}\,\mathbb{H}\big[B(a\mid b)\big] = -\int \frac{\delta B(a,b)}{\delta q}\,\log B(a\mid b)\,da\,db,
\end{equation}
i.e.\ entropies of conditionals---and hence conditional mutual informations---differentiate as if the logarithm were frozen.
\end{lemma}
\begin{proof}
$\mathbb{H}[B(a\mid b)]=-\int B(a,b)\log\frac{B(a,b)}{B(b)}$. The product-rule remainder is $-\int B(a,b)\big[\frac{\delta B(a,b)/\delta q}{B(a,b)}-\frac{\delta B(b)/\delta q}{B(b)}\big]=-\int\frac{\delta B(a,b)}{\delta q}+\int\frac{\delta B(b)}{\delta q}=0$, since $\int_a B(a,b)=B(b)$ for every $q$. \qed
\end{proof}

\begin{lemma}[Cavity back-reaction vanishes]
\label{lem:cavity}
With $\beta(\theta\mid x,y,\theta_c):=q_{\mathrm{obs},t}(\theta_A\mid x,y,\theta_c)$ independent of $q_{yx,t}$ and the exact cavity $\bar\beta(\theta\mid x,\theta_c)=\sum_y q_{yx,t}(y\mid x,\theta_c)\,\beta(\theta\mid x,y,\theta_c)$,
\begin{equation}
    \frac{\delta\,\mathcal{I}_{A,t}}{\delta q_{yx,t}(y,x,\theta_c)}=\Delta_{A,t}(x,y,\theta_c)=D_{\mathrm{KL}}\big[\beta(\cdot\mid x,y,\theta_c)\,\big\|\,\bar\beta(\cdot\mid x,\theta_c)\big],
\end{equation}
with no contribution from the dependence of $\bar\beta$ on $q_{yx,t}$. Likewise $\delta\mathcal{I}_{B,t}/\delta q_{\mathrm{trip},t}=\Delta_{B,t}$.
\end{lemma}
\begin{proof}
Write
\begin{equation*}
    \begin{aligned}
    \mathcal{I}_{A,t}=&\sum_{y,x,\theta_c}q_{yx,t}(y,x,\theta_c)\int\beta(\theta\mid x,y,\theta_c)\log\beta(\theta\mid x,y,\theta_c)\,d\theta\\
    &-\sum_{y,x,\theta_c}q_{yx,t}(y,x,\theta_c)\int\beta(\theta\mid x,y,\theta_c)\log\bar\beta(\theta\mid x,\theta_c)\,d\theta.
    \end{aligned}
\end{equation*}
The explicit derivative is $\Delta_{A,t}(x,y,\theta_c)$. Since $\sum_{y'}q_{yx,t}(y',x,\theta_c)\,\beta(\theta\mid x,y',\theta_c)$ equals $q(x,\theta_c)\,\bar\beta(\theta\mid x,\theta_c)$, the implicit derivative through $\bar\beta(\cdot\mid x,\theta_c)$ is
\begin{equation*}
    \begin{aligned}
    -q(x,\theta_c)\int\bar\beta(\theta\mid x,\theta_c)\,&\frac{\delta\log\bar\beta(\theta\mid x,\theta_c)}{\delta q_{yx,t}(y,x,\theta_c)}\,d\theta\\
    &=-q(x,\theta_c)\,\frac{\delta}{\delta q_{yx,t}(y,x,\theta_c)}\int\bar\beta(\theta\mid x,\theta_c)\,d\theta=0,
    \end{aligned}
\end{equation*}
because $\bar\beta$ is normalized for every $q$. \qed
\end{proof}

Under the cavity projection of Section~\ref{sec:family} the weight in the integral is the exact mixture while the logarithm is that of $\mathrm{Dir}(\alpha+\bar{n}_{x\theta_c})$, whose $\bar{n}$-derivative is $\log\theta_k-\psi(\alpha_k+\bar{n}_k)+\psi(\alpha_0+1)$; the residual is therefore proportional to $\mathbb{E}_{\bar\beta}[\log\theta_k]-\mathbb{E}_{\mathrm{Dir}(\alpha+\bar{n})}[\log\theta_k]=(1-r_k)\psi(\alpha_k)+r_k\psi(\alpha_k+1)-\psi(\alpha_k+r_k)=O(\alpha_k^{-2})$ by concavity of $\psi$, and is dropped. Only explicit appearances of $q$ survive, which is why the updates stay local and the mixture structure of the cavity never enters the messages.

\begin{proposition}[Observation block]
\label{prop:app_obs}
Varying~\eqref{eq:app_lagrangian} in $q_{yx}(y_1,x_1,\theta_c)$: the base~\eqref{eq:base_F} gives $\log q_{yx}+1-\log m_A$; $+\mathbb{H}_{q_{yx}}[y_1\mid x_1]$ gives $-\log q(y_1\mid x_1)$ and $-2\mathcal{I}_c$ gives $-2\log[q(\theta_c\mid x_1,y_1)/q(\theta_c\mid x_1)]$ (Lemma~\ref{lem:score}); $-2\mathcal{I}_{A}$ gives $-2\Delta_A(x_1,y_1,\theta_c)$ (Lemma~\ref{lem:cavity}); the multipliers give $\lambda_{yx}+\lambda_{y_1}+\lambda^{\mathrm{obs}}_{x_1}+\lambda^{\mathrm{obs}}_{\theta_c}$. Using $q(\theta_c\mid x,y)/q(\theta_c\mid x)=q(y\mid x,\theta_c)/q(y\mid x)$, at stationarity
\begin{equation}
    \label{eq:app_obs_stationary}
    q_{yx}^*(y_1,x_1,\theta_c) \propto \underbrace{m_A(y_1\mid x_1,\theta_c)}_{\text{predictive}}\;\underbrace{\frac{q^*(y_1\mid x_1,\theta_c)^2}{q^*(y_1\mid x_1)}}_{\text{salience}}\;\underbrace{e^{\,2\,\Delta_A(x_1,y_1,\theta_c)}}_{\text{novelty}\,\ge\,1}\;e^{-\lambda^{\mathrm{obs}}},
\end{equation}
with $\lambda^{\mathrm{obs}}:=\lambda_{y_1}+\lambda_{x_1}^{\mathrm{obs}}+\lambda_{\theta_c}^{\mathrm{obs}}$.
\end{proposition}

\begin{proposition}[Dynamics block]
\label{prop:app_dyn}
Varying in $q_{\mathrm{trip}}(x_1,x_0,u_1,\theta_c)$: the base gives $\log q_{\mathrm{trip}}+1-\log m_B$; $-\mathbb{H}_q[x_1\mid x_0,u_1]$ gives $+\log q(x_1\mid x_0,u_1)$; $+\mathbb{H}_q[u_1\mid x_0]$ gives $-\log q(u_1\mid x_0)$; $-\mathcal{I}_B$ gives $-\Delta_B$. At stationarity,
\begin{equation}
    \label{eq:app_dyn_stationary}
    q_{\mathrm{trip}}^*(x_1,x_0,u_1,\theta_c) \propto m_B(x_1\mid x_0,u_1,\theta_c)\;\frac{q^*(u_1\mid x_0)}{q^*(x_1\mid x_0,u_1)}\;e^{\,\Delta_B(x_1,x_0,u_1,\theta_c)}\;e^{-\lambda^{\mathrm{dyn}}},
\end{equation}
with $\lambda^{\mathrm{dyn}}:=\lambda_{x_1}^{\mathrm{dyn}}+\lambda_{x_0}+\lambda_{u_1}+\lambda_{\theta_c}^{\mathrm{dyn}}$.
\end{proposition}

Propositions~\ref{prop:app_obs}--\ref{prop:app_dyn} are Proposition~\ref{prop:kernels}: the factor functions (13a)--(13b) of~\cite{nuijten2026whattypeinference} with the point likelihood replaced by the posterior predictive likelihood and the information-gain weights appended. Their observation factor function carries the ratio $q(y\mid x,\theta)^2/q(y\mid x)$ for the full $\theta$; with $\theta=(\theta_A,\theta_c)$ it factorizes as
\begin{equation}
    \label{eq:app_channel_expectation}
    \begin{gathered}
    \frac{q(y\mid x,\theta_A,\theta_c)^2}{q(y\mid x)}
    =\frac{q(y\mid x,\theta_c)^2}{q(y\mid x)}\cdot\Big[\frac{q(y\mid x,\theta_A,\theta_c)}{q(y\mid x,\theta_c)}\Big]^2,\\
    \mathbb{E}_{q_{\mathrm{obs},t}(\theta_A\mid x,y,\theta_c)}\log\frac{q(y\mid x,\theta_A,\theta_c)}{q(y\mid x,\theta_c)}=\Delta_{A,t}(x,y,\theta_c),
    \end{gathered}
\end{equation}
since the second ratio equals $q_{\mathrm{obs},t}(\theta_A\mid x,y,\theta_c)/q_{\mathrm{obs},t}(\theta_A\mid x,\theta_c)$ by Bayes' rule. The finite factor is stored on the $\theta_c$-axis and retains its channel form; the $\theta_A$-dependent factor, not polynomial in $\theta_A$, is replaced by its posterior expectation, a scalar per entry, as conjugacy requires (Appendix~\ref{app:family}). Because each correction is differentiated under the measure it is defined on, the weights have the form $+D_{\mathrm{KL}}\ge0$.

\paragraph{Singleton beliefs.} $\Delta F$ has no singleton dependence, so as in~\cite{nuijten2026whattypeinference} (D.4.6): for the degree-two variables the singleton entropy cancels against the counting term, and for the shared variables only the sum of multipliers is fixed,
\begin{equation}
    \label{eq:app_multipliers}
    \begin{gathered}
    \lambda_{y_1} = -\log\hat{p}_y, \qquad \lambda_{x_0} = -\log p(x_0), \qquad \lambda_{u_1} = -\log p(u_1),\\
    \lambda_{x_1}^{\mathrm{obs}} + \lambda_{x_1}^{\mathrm{dyn}} = -\log q_{x_1}^* - 1 - \log\hat{p}_x,\\
    \lambda_{\theta_c}^{\mathrm{obs}} + \lambda_{\theta_c}^{\mathrm{dyn}} = -\log q_{\theta_c}^* - 1 - \log p(\theta_c).
    \end{gathered}
\end{equation}

\begin{remark}[Parameter beliefs and the frozen-parameter update]
\label{rem:app_theta}
During planning $q(\theta_A),q(\theta_B)$ are frozen, so $m$, the one-count posteriors and the weights depend on them only through the fixed concentrations, whereas $q(\theta_c)$ is an ordinary singleton belief and moves within the sweep. Exact functional stationarity of $F_T$ in a free $q_{\theta_A}$ has two parts: the conjugate statistic (a Dirichlet count update) and a \emph{centered score} $s_A(\theta_A)=\mathbb{E}_{w}[\log\beta-\log\bar\beta-\Delta_A]$ under the reweighted block measure $w\propto q_{yx,t}\,p(y\mid x,\theta_A,\theta_c)/m_A$, which integrates to zero against $q_{\theta_A}$ and vanishes at point mass and as counts grow. The score would reshape $q_{\theta_A}$ toward parameter values of above-average posterior distinguishability without shifting count mass. Between environment steps we drop it and update the concentrations by the plain conjugate count update from the realized transition and observation,
\begin{equation}
    \label{eq:app_theta_T}
    \begin{gathered}
    q_{\theta_A} \propto p(\theta_A)\,\exp\Big(\textstyle\sum_{t}\mathbb{E}_{q_{yx,t}}[\log p(y_t\mid x_t,\theta_A,\theta_c)]\Big),\\
    q_{\theta_B} \propto p(\theta_B)\,\exp\Big(\textstyle\sum_{t}\mathbb{E}_{q_{\mathrm{trip},t}}[\log p(x_t\mid x_{t-1},u_t,\theta_B,\theta_c)]\Big),
    \end{gathered}
\end{equation}
with $q_{yx,t},q_{\mathrm{trip},t}$ now the (soft) realized data, the count distributed over contexts according to $q(\theta_c)$, which is itself updated by exact Bayes. This is a partial-gradient step of $F_T$ that preserves exact conjugacy, with the dropped terms exhibited; the score is worth including only if monitoring shows the plain step increasing $F_T$.
\end{remark}

\subsection{Solving the Stationarity Equations as Message Passing}
\label{app:mp_updates}

Given the factor functions~\eqref{eq:kernel_obs}--\eqref{eq:kernel_dyn} with lagged conditionals $r_t$, Propositions~\ref{prop:app_obs}--\ref{prop:app_dyn} are an exact Bethe belief-propagation problem on the discrete graph over $\{y_t,x_t,u_t,\theta_c\}$~\cite{senoz2021variational}, and the standard factor-to-variable rule applies,
\begin{equation}
    \label{eq:app_msg_rule}
    \mu_{a\to j}(s_j)\propto\sum_{s_{a\setminus j}}\bar{g}_a(s_a)\prod_{i\in\mathcal{E}(a)\setminus j}\mu_{i\to a}(s_i).
\end{equation}
Observation factor, writing $\mu_{y_t},\mu_{x_t},\mu_{\theta_c}$ for the messages into $f_{\mathrm{obs},t}$ (with $\mu_{y_t}=\hat{p}_y$):
\begin{equation}
    \label{eq:app_msg_obs}
    \begin{gathered}
    \mu_{\mathrm{obs}_t\to x_t}=\sum_{y_t,\theta_c}\bar{g}_{\mathrm{obs},t}\,\mu_{y_t}\mu_{\theta_c},\qquad
    \mu_{\mathrm{obs}_t\to y_t}=\sum_{x_t,\theta_c}\bar{g}_{\mathrm{obs},t}\,\mu_{x_t}\mu_{\theta_c},\\
    \mu_{\mathrm{obs}_t\to\theta_c}=\sum_{y_t,x_t}\bar{g}_{\mathrm{obs},t}\,\mu_{y_t}\mu_{x_t}.
    \end{gathered}
\end{equation}
Dynamics factor: analogously to $x_t$, $x_{t-1}$, $u_t$, $\theta_c$, each summing $\bar{g}_{\mathrm{dyn},t}$ against the other three incoming messages, with $\mu_{u_t\to\mathrm{dyn}_t}=p(u_t)$ and boundary $\mu_{x_0\to\mathrm{dyn}_1}=p(x_0)$. There are no messages to or from $\theta_A^{(t)},\theta_B^{(t)}$. Block and singleton beliefs are
\begin{equation}
    \label{eq:app_block_beliefs}
    \begin{gathered}
    q_{yx,t}\propto\bar{g}_{\mathrm{obs},t}\,\mu_{y_t\to\mathrm{obs}_t}\,\mu_{x_t\to\mathrm{obs}_t}\,\mu_{\theta_c\to\mathrm{obs}_t},\\
    q_{\mathrm{trip},t}\propto\bar{g}_{\mathrm{dyn},t}\,\mu_{x_t\to\mathrm{dyn}_t}\,\mu_{x_{t-1}\to\mathrm{dyn}_t}\,\mu_{u_t\to\mathrm{dyn}_t}\,\mu_{\theta_c\to\mathrm{dyn}_t},
    \end{gathered}
\end{equation}
\begin{equation}
    \label{eq:app_singleton_beliefs}
    \begin{gathered}
    q(x_t)\propto\hat{p}_x\,\mu_{\mathrm{obs}_t\to x_t}\,\mu_{\mathrm{dyn}_t\to x_t}\,\mu_{\mathrm{dyn}_{t+1}\to x_t},\qquad
    q(\theta_c)\propto p(\theta_c)\prod_{\tau=1}^{T}\mu_{\mathrm{obs}_\tau\to\theta_c}\,\mu_{\mathrm{dyn}_\tau\to\theta_c},\\
    q(u_t)\propto p(u_t)\,\mu_{\mathrm{dyn}_t\to u_t},\qquad
    q(y_t)\propto\hat{p}_y\,\mu_{\mathrm{obs}_t\to y_t},
    \end{gathered}
\end{equation}
which at a fixed point recover~\eqref{eq:app_obs_stationary}--\eqref{eq:app_dyn_stationary} with~\eqref{eq:app_multipliers}. The full factor beliefs follow from~\eqref{eq:factor_belief}, $q_{\mathrm{obs},t}=q_{yx,t}\,\mathrm{Dir}(\alpha_A+e_{y_tx_t\theta_c})$, and are never stored. The block-layer cost is independent of the size of the Dirichlet parameter space and linear in $K$.

\paragraph{Refresh.} After each sweep, from the new block beliefs,
\begin{equation}
    \label{eq:app_refresh}
    \begin{gathered}
    r_t(y_t\mid x_t,\theta_c)\leftarrow q_{yx,t}(y_t\mid x_t,\theta_c),\qquad
    r_t(y_t\mid x_t)\leftarrow q_{yx,t}(y_t\mid x_t),\\
    r_t(u_t\mid x_{t-1})\leftarrow q_{\mathrm{trip},t}(u_t\mid x_{t-1}),\qquad
    r_t(x_t\mid x_{t-1},u_t)\leftarrow q_{\mathrm{trip},t}(x_t\mid x_{t-1},u_t),
    \end{gathered}
\end{equation}
\begin{equation}
    \label{eq:app_refresh_tilts}
    \bar{n}_{x_t\theta_c}\leftarrow r_t(\cdot\mid x_t,\theta_c),\quad
    \Delta_{A,t}(x_t,y_t,\theta_c)\leftarrow D_{\mathrm{KL}}\big[\mathrm{Dir}(\alpha_A+e_{y_tx_t\theta_c})\,\big\|\,\mathrm{Dir}(\alpha_A+\bar{n}_{x_t\theta_c})\big],
\end{equation}
and likewise $\Delta_{B,t}$ with fractional counts $\bar{n}_{x_{t-1}\theta_c}(x_t,u_t)=q_{\mathrm{trip},t}(x_t,u_t\mid x_{t-1},\theta_c)$ spread over the $(x_{t-1},u_t,\theta_c)$ columns, so that the divergence carries cross-column terms. The factor functions are rebuilt immediately. Damping is applied linearly in $\bar{n}$, which keeps the cavity Dirichlet. Since no conditional is built from messages and the one-count posteriors never feed back into themselves, the divergent channel recursion of~\cite{nuijten2026whattypeinference} is structurally absent.

\paragraph{Closed forms.} With $q_{\theta_A}=\prod_{x,\theta_c}\mathrm{Dir}(\alpha_A(\cdot\mid x,\theta_c))$: $m_A(y\mid x,\theta_c)=\alpha_A(y\mid x,\theta_c)/\alpha_{A,0}(x,\theta_c)$; $q_{\mathrm{obs},t}(\theta_A\mid x,y,\theta_c)=\mathrm{Dir}(\alpha_A+e_{yx\theta_c})$; the projected cavity is $\mathrm{Dir}(\alpha_A+\bar{n}_{x\theta_c})$; and for two Dirichlets on one column,
\begin{equation}
    \label{eq:app_dir_kl}
    D_{\mathrm{KL}}[\mathrm{Dir}(a)\,\|\,\mathrm{Dir}(b)]=\log\frac{\Gamma(a_0)}{\Gamma(b_0)}+\sum_k\Big[\log\frac{\Gamma(b_k)}{\Gamma(a_k)}+(a_k-b_k)\big(\psi(a_k)-\psi(a_0)\big)\Big].
\end{equation}
The salience $I_{q_{yx,t}}(y_t;\theta_c\mid x_t)$ is a finite sum over the stored array. $F_T$~\eqref{eq:F_T} is assembled from quantities the sweep already computes and gives a stopping criterion on $|\Delta F_T|$.

\paragraph{Schedule.} Algorithm~\ref{alg:planner} of the main text, with the objects of this appendix in place: the predictive likelihoods and the one-count posteriors $\mathrm{Dir}(\alpha_A+e_{yx\theta_c})$, $\mathrm{Dir}(\alpha_B+e_{x'xu\theta_c})$ are built once per planning call from the current counts; each sweep rebuilds the factor functions~\eqref{eq:kernel_obs}--\eqref{eq:kernel_dyn}, runs one forward--backward pass of~\eqref{eq:app_msg_rule}--\eqref{eq:app_singleton_beliefs} over the chain, applies the refresh~\eqref{eq:app_refresh}--\eqref{eq:app_refresh_tilts} and evaluates $F_T$; the policy is read from $q(u_1\mid x_0)$ at convergence. $F_T$ is a monitoring quantity, not a certificate: block-coordinate descent on a non-convex functional with lagged factors is not guaranteed monotone, and period-two orbits appear as non-monotone $F_T$, which is the criterion by which damping is selected. Uniform $r_t$, zero weights and flat parameter beliefs recover plain belief propagation.

\subsection{Generic Scheme for Arbitrary $T$}
\label{app:mp_generic}

The corrections~\eqref{eq:deltaF} are additive in $t$ and every object of Section~\ref{sec:family} is per-factor, so the family carries one block belief per factor, one shared context belief, and two global, frozen Dirichlet beliefs,
\begin{equation}
    \label{eq:app_mf_T}
    \{q_{yx,t},\,q_{\mathrm{trip},t}\}_{t=1}^{T},\quad\text{singletons incl.\ }q(\theta_c),\quad q_{\theta_A}=\mathrm{Dir}(\alpha_A),\ q_{\theta_B}=\mathrm{Dir}(\alpha_B).
\end{equation}
The one-count posteriors and $m_A,m_B$ are time-independent; the cavities, lagged conditionals and weights are time-indexed. Lemmas~\ref{lem:score}--\ref{lem:cavity} apply per step and give the factor functions~\eqref{eq:kernel_obs}--\eqref{eq:kernel_dyn} at every $t$. The graph is the temporal chain of~\cite{nuijten2026whattypeinference} with $\theta_c$ attached to every factor, so the sweep is loopy belief propagation as there: with state potential $\phi_t(x_t)=\hat{p}_x(x_t)\,\mu_{\mathrm{obs}_t\to x_t}(x_t)$ and context messages $\mu_{\theta_c\to a}=q(\theta_c)/\mu_{a\to\theta_c}$,
\begin{align}
    \overrightarrow{m}_t(x_t) &\propto \sum_{x_{t-1},u_t,\theta_c} \bar{g}_{\mathrm{dyn},t}\,p(u_t)\,\mu_{\theta_c\to\mathrm{dyn}_t}\,\phi_{t-1}(x_{t-1})\,\overrightarrow{m}_{t-1}(x_{t-1}), \label{eq:app_fwd}\\
    \overleftarrow{m}_t(x_t) &\propto \sum_{x_{t+1},u_{t+1},\theta_c} \bar{g}_{\mathrm{dyn},t+1}\,p(u_{t+1})\,\mu_{\theta_c\to\mathrm{dyn}_{t+1}}\,\phi_{t+1}(x_{t+1})\,\overleftarrow{m}_{t+1}(x_{t+1}), \label{eq:app_bwd}
\end{align}
with $\overrightarrow{m}_0=p(x_0)$, $\phi_0\equiv1$, $\overleftarrow{m}_T\equiv1$; then $q(x_t)\propto\phi_t\,\overrightarrow{m}_t\,\overleftarrow{m}_t$, the block beliefs follow from~\eqref{eq:app_block_beliefs} with $\mu_{x_t\to\mathrm{obs}_t}=\hat{p}_x\overrightarrow{m}_t\overleftarrow{m}_t$, $\mu_{x_{t-1}\to\mathrm{dyn}_t}=\phi_{t-1}\overrightarrow{m}_{t-1}$, $\mu_{x_t\to\mathrm{dyn}_t}=\phi_t\overleftarrow{m}_t$, and $q(\theta_c)$ from~\eqref{eq:app_singleton_beliefs}. At $T{=}1$ this reduces to~\eqref{eq:app_msg_obs}--\eqref{eq:app_singleton_beliefs}. Because every copy is drawn i.i.d.\ from the same frozen $\mathrm{Dir}(\alpha)$, novelty is scored per factor against the same belief, with no within-horizon diminishing returns; replanning after every real step bounds this to one horizon, and the causal alternative $\alpha_t=\alpha+\sum_{s<t}\bar{n}_s$ (assumed-density filtering along the horizon) retains conjugacy but makes the posteriors and $m$ per-sweep objects. Salience has no such issue: $\theta_c$ is shared, so a cue at step $s$ updates $q(\theta_c)$ for every step.

\subsection{Limits}
\label{app:mp_limit}

All Dirichlet-dependent quantities are available in closed form, so conjugacy is exact and the block-layer cost is independent of the Dirichlet parameter-space size. As $q_{\theta_A},q_{\theta_B}\to\delta_{\theta^*}$, posterior and cavity coincide, both weights vanish pointwise, $m\to p(\cdot\mid\theta^*,\theta_c)$, and~\eqref{eq:kernel_obs}--\eqref{eq:kernel_dyn} become the EFE factor functions (13a)--(13b) of~\cite{nuijten2026messagepassingrealizationexpected,nuijten2026whattypeinference} with their enumerable $\theta$ read as $\theta_c$: inference over $K$ known models is the point-mass limit of learning $K$ matrix pairs. The approach is continuous, $D_{\mathrm{KL}}[\mathrm{Dir}(\alpha{+}e)\,\|\,\mathrm{Dir}(\alpha{+}\bar{n})]\to0$ as counts grow. At $K=1$ the salience ratio is identically one and the scheme reduces to Dirichlet learning alone, the setting of Section~\ref{sec:experiments}. The two epistemic drives differ in scale: novelty is $O(1/\alpha_0)$ and vanishes as counts grow, salience is scale-free and vanishes only when $\theta_c$ is resolved; novelty couples no columns, salience couples all of them through the single shared $\theta_c$.

\section{Experimental Details}
\label{app:experiments}

\subsection{Environments and Protocol}
\label{app:exp_setup}

\paragraph{Shared protocol.} Four cardinal moves (N, E, S, W; walls and grid edges bump), $K=1$, one unknown block initialized as a flat Dirichlet with mass $0.1$ per entry, the known block clamped by counts of $10^{6}$ times the true matrix. Every planner runs the same exact filter under the posterior-mean matrices $m_A,m_B$ and the count update of Algorithm~\ref{alg:planner}, line~10. Planning: $100$ sweeps at damping $0.5$, uniform $\hat{p}_y$, horizon $H_s=\min(H,\,T-s+1)$ at step $s$ of a $T$-step episode as in~\cite{nuijten2026messagepassingrealizationexpected}, argmax of $q(u_1)$; counts persist across episodes, pose and belief reset, seeds are shared across planners. EFE enumerates the $4^{H}$ open-loop policies under $m_A,m_B$ and takes the argmax of $\exp(-G)$, $G$ being the per-policy sum of~\eqref{eq:efe}.

\paragraph{Fog-walls maze.} $5\times5$ cells $(x,y)$, $y$ upwards; start $(0,2)$, goal $(4,2)$. $A$ is the identity except on the fog block $(1,1)$, $(1,2)$, $(1,3)$, $(2,2)$, $(3,2)$, whose columns are uniform over the block, and the noisy cell $(2,0)$ (mass $0.2$ on itself, $0.8$ over its neighbours). $B$ is deterministic except at the drift cell $(1,2)$: intended move w.p.\ $0.4$, north or south w.p.\ $0.3$ each. Cells $(2,1)$, $(2,3)$, $(3,1)$, $(3,3)$ are walls: the direct 4-step route crosses drift and fog, the 8-step top and bottom rows do not (the latter passes the noisy cell). $T=H=8$; episodes end at the goal. Learn-$A$: all 25 columns of $A$ flat, $B$ known; the matrix error is the mean of $|m_A-A|$ over all entries. This is the fog maze of~\cite{nuijten2026messagepassingrealizationexpected} with its four sinks replaced by walls.

\paragraph{Unmapped detour, doors variant.} $5\times4$ grid ($N=20$), start $(0,1)$, goal $(4,2)$, identity $A$, walls $(1,2)$, $(2,2)$, $(3,2)$. The shortcut $\to\to\to\to\uparrow$ ends at the ice $(4,1)$, whence NORTH reaches the goal w.p.\ $p=0.5$ and otherwise slips into the absorbing dead end $(4,0)$. The detour $\uparrow\uparrow\to\to\to\to\downarrow$ is deterministic through $(0,2)$, $(0,3)$, $(1,3)$, $(2,3)$, $(3,3)$, $(4,3)$; exactly its six $(\mathrm{cell},\mathrm{action})$ columns start flat, the other actions at those cells (bumps or the step back) are known. $T=8$, $H=7$; episodes end at the goal or the dead end, the last transition folded into the counts. \emph{Detour} = goal reached with every visited cell on the detour; $\mathrm{err}_B$ = mean $|m_B-B|$ over the six door columns (prior value $0.095$). Frozen $B$ keeps the door columns flat and never updates them; the \emph{cells} variant flattens all four actions of the six cells.

\subsection{Controls and Robustness}
\label{app:exp_controls}

\paragraph{Fog-walls, $A,B$ known} (30 seeds). At $H=8$ every SP-EFE variant puts $q(u_1)$ entirely on NORTH and takes the top row on 30/30 seeds; SP-Pref-Plan ($q(\mathrm{EAST})=0.73$) and SP-Pref ($0.32$) go into the fog and score 18/30 and 17/30: they plan as if they will know where they land, and the planning correction sharpens that commitment. At $H=11$ bumping after a wrong landing still reaches the goal and the ceiling closes to 29/30 (SP-Pref 22/30). The learn-$A$ policies of SP-Pref(-Plan) equal their known-$A$ policies: their goal rate is the luck of the drift, not a function of $m_A$.

\paragraph{Unmapped detour.} Known $B$: SP-Pref ($q(\mathrm{EAST})=0.30$) takes the ice because the uncorrected Bethe term $-\mathbb{H}[q(u_t\mid x_{t-1})]$ rewards states in which all actions look alike, which absorbing cells do: a sink costs $4-\log4\approx2.6$ nats per step instead of $4$. Frozen $B$: SP-Pref-Plan gambles ($q(\mathrm{EAST})=0.70$, zero door visits); $\text{SP-EFE}^{-\mathrm{dyn}}$ farms the first door (0 goals: $\Delta_B$ never decays); SP-EFE, $\text{SP-EFE}^{-\Delta_B}$ and SP-Pref reach the goal via the detour on 300/300 episodes \emph{without learning}: with frozen flat doors ``most uncertain column'' and ``detour'' are the same action at every cell, so any entropy follower is walked along the route ($^\dagger$ in Table~\ref{tab:results}).

\paragraph{Robustness.} Action indices permuted (W, S, E, N; 200 episodes): $\text{SP-EFE}^{-\mathrm{dyn}}$ 198, SP-Pref 101 (25 detours, then the ice), SP-Pref-Plan 89. Mirrored $6\times4$ layout, 7 doors, $H=8$ (300 episodes $\times$ 3 seeds): SP-EFE 300, $\text{SP-EFE}^{-\mathrm{dyn}}$ 299, SP-Pref $200\pm6$, SP-Pref-Plan $150\pm5$, EFE 100/100. Cells variant (200 episodes $\times$ 5 seeds): $\text{SP-EFE}^{-\mathrm{dyn}}$ 159, SP-Pref-Plan $97.6\pm12.5$ with $m_B$ bit-identical to its prior, SP-Pref 185 (off-path flat columns pay its bonus forever), SP-EFE 35.

\AtEndDocument{\typeout{get arXiv to do 4 passes: Label(s) may have changed. Rerun}}
\end{document}